\documentclass[10pt, journal, compsoc]{IEEEtran}
\usepackage{amsmath,amsfonts}
\usepackage{amssymb}
\usepackage{algorithmic}
\usepackage{array}
\usepackage[caption=false,font=normalsize,labelfont=sf,textfont=sf]{subfig}
\usepackage{textcomp}
\usepackage{stfloats}
\usepackage{url}
\usepackage{verbatim}
\usepackage{graphicx}
\usepackage{balance}
\usepackage{cite}
\usepackage{amsthm}
\usepackage{epstopdf}
\usepackage{multirow}
\usepackage{booktabs}
\usepackage{diagbox}
\usepackage{bm}
\usepackage[ruled,lined,boxed]{algorithm2e}
\usepackage{enumitem}

\newtheorem{theorem}{Theorem}

\usepackage{xcolor}
\begin{document}

\title{Cross-view Joint Learning for Mixed-Missing Multi-view Unsupervised Feature Selection}

\author{Zongxin Shen, Yanyong Huang, Dongjie Wang, Jinyuan Chang, Fengmao~Lv, Tianrui Li,~\IEEEmembership{Senior Member,~IEEE}, and Xiaoyi Jiang,~\IEEEmembership{Senior Member,~IEEE}

\thanks{Zongxin~Shen, Yanyong~Huang, and Jinyuan Chang are with the Joint Laboratory of Data Science and Business Intelligence, School of Statistics and Data Science, Southwestern University of Finance and Economics, Chengdu 611130, China (e-mail: zxshen@smail.swufe.edu.cn; huangyy@swufe.edu.cn;  changjinyuan@swufe.edu.cn), Yanyong Huang is the corresponding author;}

\thanks{Dongjie Wang is with the Department of Electrical Engineering and Computer Science, University of Kansas, Lawrence, KS 66045, USA (e-mail: wangdongjie@ku.edu);}

\thanks{Fengmao~Lv and Tianrui Li are with the School of Computing and Artificial Intelligence, Southwest Jiaotong University, Chengdu 611756, China (e-mail: fengmaolv@126.com; trli@swjtu.edu.cn);}

\thanks{Xiaoyi Jiang is with the Faculty of Mathematics and Computer Science, University of Münster, Münster 48149, Germany (e-mail: xjiang@uni-muenster.de).}}

\markboth{Journal of \LaTeX\ Class Files,~Vol.~14, No.~8, August~2021}%
{Shell \MakeLowercase{\textit{et al.}}: A Sample Article Using IEEEtran.cls for IEEE Journals}

\IEEEpubid{0000--0000/00\$00.00~\copyright~2021 IEEE}

\maketitle

\begin{abstract}
Incomplete multi-view unsupervised feature selection (IMUFS), which aims to identify representative features from unlabeled multi-view data containing missing values, has received growing attention in recent years. Despite their promising performance, existing methods face three key challenges: 1) by focusing solely on the view-missing problem, they are not well-suited to the more prevalent mixed-missing scenario in practice, where some samples lack entire views or only partial features within views; 2) insufficient utilization of consistency and diversity across views limits the effectiveness of feature selection; and 3) the lack of theoretical analysis makes it unclear how feature selection and data imputation interact during the joint learning process. Being aware of these, we propose CLIM-FS, a novel IMUFS method designed to address the mixed-missing problem. Specifically, we integrate the imputation of both missing views and variables into a feature selection model based on nonnegative orthogonal matrix factorization, enabling the joint learning of feature selection and adaptive data imputation. Furthermore, we fully leverage consensus cluster structure and cross-view local geometrical structure to enhance the synergistic learning process. We also provide a theoretical analysis to clarify the underlying collaborative mechanism of CLIM-FS. Experimental results on eight real-world multi-view datasets demonstrate that CLIM-FS outperforms state-of-the-art methods.
\end{abstract}


\section{Introduction}
\IEEEPARstart{M}{ulti-view} unsupervised feature selection (MUFS) is a crucial dimensionality reduction technique designed to identify a representative subset of features from unlabeled multi-view data, thereby alleviating the ``curse of dimensionality'' and improving the performance of downstream tasks~\cite{R.ZhangIF2019,HouCTPAMI2023,G.J.LiTKDE2024}. Existing MUFS methods rely heavily on the assumption that all views of data are fully observed. However, this assumption is often violated in practice, as multi-view datasets frequently contain missing values due to various factors, such as equipment malfunctions or human errors~\cite{J.J.TangASC2024, P.X.ZengTPAMI2023}. The presence of missing values not only makes traditional MUFS methods that require complete data unsuitable, but also distorts the intrinsic data structure and weakens both intra-view and inter-view correlations, ultimately diminishing the effectiveness of these approaches. Consequently, selecting informative features from unlabeled multi-view data with missing values remains a significant challenge in real-world applications.

In recent years, various incomplete multi-view unsupervised feature selection (IMUFS) methods have been proposed to address the above problem. These methods can be classified into two categories. The first category, referred to as the ``two-stage'' approach, first imputes the missing data using an imputation algorithm and then performs feature selection on the imputed dataset using conventional MUFS methods, such as Cross-view Locality Preserved Diversity and Consensus Learning (CvLP\_DCL)~\cite{CvLP_DCL}, Unsupervised Kernel-based Multi-view Feature selection (UKMFS)~\cite{R.Y.HuAAAI2025}, and Partition-level Tensor Learning-based Multiview Unsupervised Feature Selection (PTFS)~\cite{PTFS}. However, this kind of method treats feature selection and data imputation as separate processes, overlooking the potential synergy between them. Local structure information obtained from the feature selection process can guide the imputation of missing data, while improved imputations can, in turn, enhance the effectiveness of feature selection. 

Rather than treating feature selection and data imputation as separate processes, the second category of methods, known as ``one-stage,'' integrates both into a unified learning framework. One such method is Unified View Imputation and Feature Selection Learning (UNIFIER)~\cite{UNIFIER}, in which sample-level and feature-level local structures are leveraged to guide the imputation process, and the imputed data further improves local structure learning for feature selection. Additionally, Huang et al. propose Tensorial Incomplete Multi-view Unsupervised Feature Selection (TIME-FS)~\cite{TIME-FS}, which learns a cross-view consistency anchor graph and a view-preference weight matrix to jointly guide both data imputation and feature selection. Yang et al. combine self-representation learning with a tensor-based low-rank constraint to adaptively impute missing views during the feature selection process~\cite{TERUIMUFS}. By enabling collaborative learning between data imputation and feature selection, this ``one-stage'' approach has demonstrated promising results in feature selection.

Despite the significant progress made by existing IMUFS methods, they primarily address a specific type of incomplete multi-view scenario, known as the ``view-missing'' problem, where entire views are absent for certain samples. However, in real-world applications, multi-view data often encounters the more general “mixed-missing” issue, where some samples completely lack certain views or have only partial features missing within specific views. For instance, in object recognition tasks, data captured from different camera angles constitute distinct views, and certain camera perspectives may be entirely unavailable for some objects due to network failures. Additionally, some camera views may have partially missing variables for certain objects as a result of temporary occlusion or motion blur. Fig.~\ref{Fig1} illustrates three types of incomplete multi-view data scenarios: view-missing, variable-missing (where samples have only partially missing variables within views), and mixed-missing. The mixed-missing issue represents a general pattern of missing data, with view-missing and variable-missing serving as two special cases. Although existing approaches effectively address the view-missing issue by jointly performing feature selection and view imputation, they still require pre-imputation of missing variables in mixed-missing or variable-missing scenarios, which leads to suboptimal feature selection performance. Furthermore, current IMUFS methods do not fully leverage cross-view consistency and diversity to guide the joint learning of feature selection and data imputation, limiting their overall effectiveness. In addition, these methods lack a theoretical analysis of the collaborative learning mechanism between feature selection and data imputation, making it difficult to interpret their models clearly. Specifically, it remains unclear how and why the interaction between feature selection and data imputation leads to improved feature selection performance.

\begin{figure}[t]
\centering
\includegraphics[width=\columnwidth]{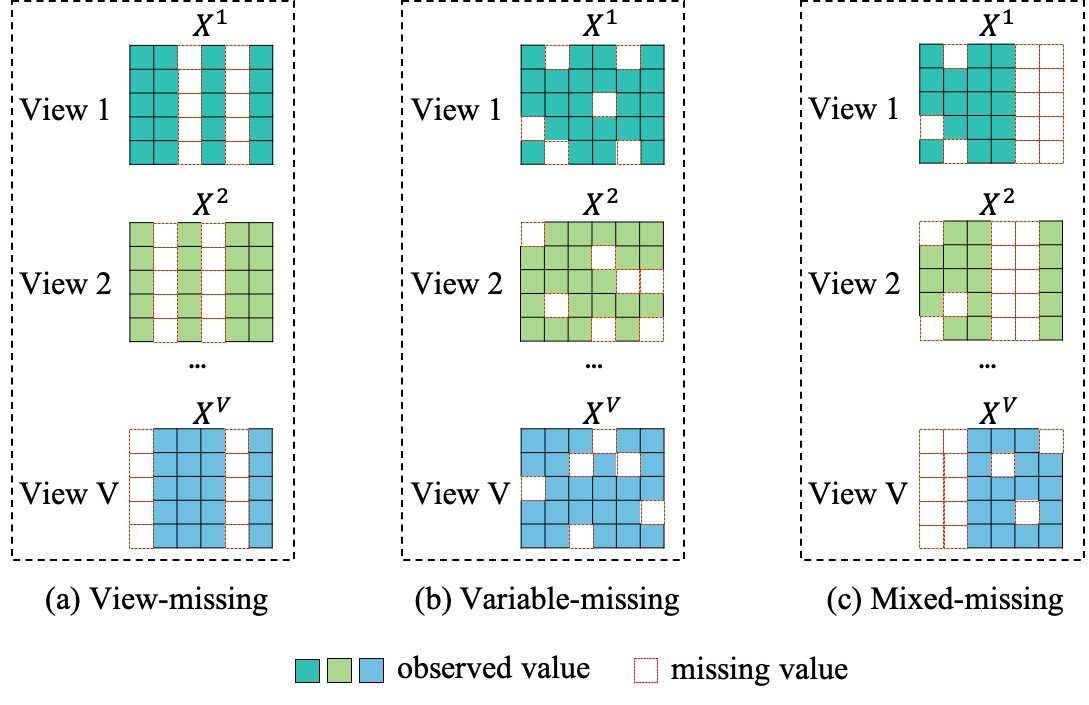} 
\caption{Illustration of three types of incomplete multi-view data scenarios: (a) view-missing, (b) variable-missing, and (c) mixed-missing. $\bm{X}^{v}$ denotes the data matrix of the $v$-th view, where each column represents a sample and each row corresponds to a feature.}
\label{Fig1}
\end{figure}

To address the above issues,  we propose a novel IMUFS method in this paper, called Cross-view joint Learning for mIxed-missing Multi-view unsupervised Feature Selection (CLIM-FS). Specifically, to address the mixed-missing problem, we first incorporate both missing views and missing variables as optimization variables into the nonnegative orthogonal matrix factorization-based feature selection model, thereby enabling joint feature selection and data imputation. Furthermore, the proposed CLIM-FS simultaneously exploits both the consensus cluster structure and the cross-view local geometric structure derived from the feature selection process to guide data imputation. Additionally, CLIM-FS ensures that neighboring samples remain close after imputation while maintaining separation from others, which in turn enhances feature selection performance. The framework of CLIM-FS is shown in Fig.~\ref{framework}. We also present a theoretical analysis to provide deeper insight into the collaborative learning mechanism underlying CLIM-FS.  In addition, we develop an alternative iterative optimization algorithm that is theoretically guaranteed to converge for solving the proposed model. Finally, we conduct extensive experiments on eight real-world multi-view datasets to demonstrate the effectiveness of our method compared with several state-of-the-art approaches. The main contributions of this paper are summarized as follows:
\begin{enumerate}
\item To the best of our knowledge, this is the first work to address multi-view unsupervised feature selection in the general mixed-missing scenario,  which advances MUFS toward a more realistic setting.

\item We integrate feature selection and data imputation into a joint learning framework, simultaneously exploiting cross-view consensus and diversity information to enhance their synergy, thereby improving the performance of feature selection.

\item We present a comprehensive theoretical analysis of the collaborative learning mechanism underlying CLIM-FS, along with the introduction of an effective optimization algorithm for our model and a demonstration of its theoretical convergence.

\item We conduct extensive experiments on eight real-world multi-view datasets to evaluate the effectiveness of the proposed CLIM-FS method, and the results show that CLIM-FS demonstrates superior performance compared to state-of-the-art methods.
\end{enumerate}

The remainder of this paper is organized as follows. Section 2 reviews related work. Section 3 details the proposed CLIM-FS method, while Section 4 introduces the corresponding optimization algorithm. Section 5 provides a theoretical analysis of the collaborative learning mechanism in CLIM-FS, as well as the algorithm's convergence and computational complexity. In Section 6, extensive experiments are conducted to verify the effectiveness of the proposed method. Finally, Section 7 concludes the paper.

\section{Related Work}\label{sec:Related work}
In this section, we present a brief overview of recent studies on multi-view unsupervised feature selection and incomplete multi-view unsupervised feature selection.
\subsection{Multi-view Unsupervised Feature Selection} In the past decade, numerous MUFS methods has been developed. Tang et al. proposed a cross-view similarity graph learning model with adaptive view weights to preserve the local manifold structure of data, thereby facilitating the selection of discriminative features~\cite{CvLP_DCL}. Cao et al. employed the random walk strategy to explore the multi-order neighbor information to guide feature selection~\cite{CFSMO}. Yuan et al. adaptively learned similarity-induced graphs under tensor low-rank constraints, which can exploit the high-order consensus information across different views to identify important features~\cite{TLR_MUFS}. Fang et al. integrated multi-view unsupervised feature selection and clustering into a joint learning framework, allowing for simultaneous feature selection and clustering~\cite{JMVFG}. Cao et al. integrated partition-level tensor learning with an adaptive self-paced strategy to capture high-order view correlations and discriminative partition information for multi-view unsupervised feature selection~\cite{PTFS}. Zhang et al. constructed bipartite graphs to capture the similarity structure between samples and anchors, which helps identify important features and reduces computational costs~\cite{C.L.ZhangIJCAI2024}. Wang et al. combined multi-view spectral clustering and weighted low-rank tensor learning to generate pseudo labels for guiding feature selection~\cite{WLTL}. Xu et al. leveraged graph regularization to capture shared manifold structures and kernel-based diversity representation to characterize view-specific dependencies, thereby enhancing discriminative feature selection~\cite{GCDUFS}. Hu et al. leveraged robust self-representation to learn consistent graph structures across views, and employed binary hashing to generate weakly-supervised labels that guide the feature selection process~\cite{R.Y.HuAAAI2025}. Although the aforementioned methods have shown promising performance of feature selection, all of them implicitly assume that each view of data is completely observed. Hence, these methods cannot be directly applied to the incomplete multi-view scenario.

\subsection{Incomplete Multi-view Unsupervised Feature Selection}
To address the incompleteness issue of multi-view data, several incomplete multi-view feature selection methods have been proposed in recent years. Xu et al. first impute the missing views using mean values, and then employ weighted non-negative matrix factorization to select important features while assigning lower weights to the imputed samples, thereby mitigating their influence on feature selection~\cite{CVFS}. Huang et al. leverage the complementary information across different views to learn missing similarities among data points, thereby obtaining a complete similarity-induced graph to preserve the local manifold structure of data.~\cite{C2IMUFS}. Yang et al. mitigate the impact of missing data by adaptively assigning sample weights during graph learning and enforce a low-redundancy constraint in a low-dimensional space to identify non-redundant and representative features~\cite{UIMUFSLR}. The above methods treat feature selection and data imputation as two independent processes, neglecting the synergistic effect between them. To tackle this problem, Huang et al. proposed a unified learning framework for feature selection and missing-view imputation, named UNIFIER briefly~\cite{UNIFIER}. It can utilize the local structures of both sample space and feature space to guide data imputation while identifying discriminative features. Moreover, TIME-FS simultaneously performed missing-view imputation and unsupervised feature selection on incomplete multi-view data by leveraging a tensor-based learning framework~\cite{TIME-FS}. Yang et al. leveraged tensor-based low-rank representations to adaptively recover missing samples during the feature selection process and incorporated self-representation learning to enhance the robustness of the feature selection model against noise~\cite{TERUIMUFS}. However, as previously mentioned, existing IMUFS methods face three key challenges: (\romannumeral1) They exclusively focus on the view-missing issue, making them unsuitable for the general mixed-missing scenario, which is more commonly encountered in real-world applications. (\romannumeral2) Existing IMUFS methods fail to comprehensively exploit cross-view consistency and diversity information to guide the joint learning of feature selection and data imputation, thereby limiting their performance in feature selection. (\romannumeral3) They lack a theoretical analysis that elucidates the collaborative learning mechanism between feature selection and data imputation, which hinders a clear understanding of their models.

\begin{figure}[t]
\centering
\includegraphics[width=\columnwidth]{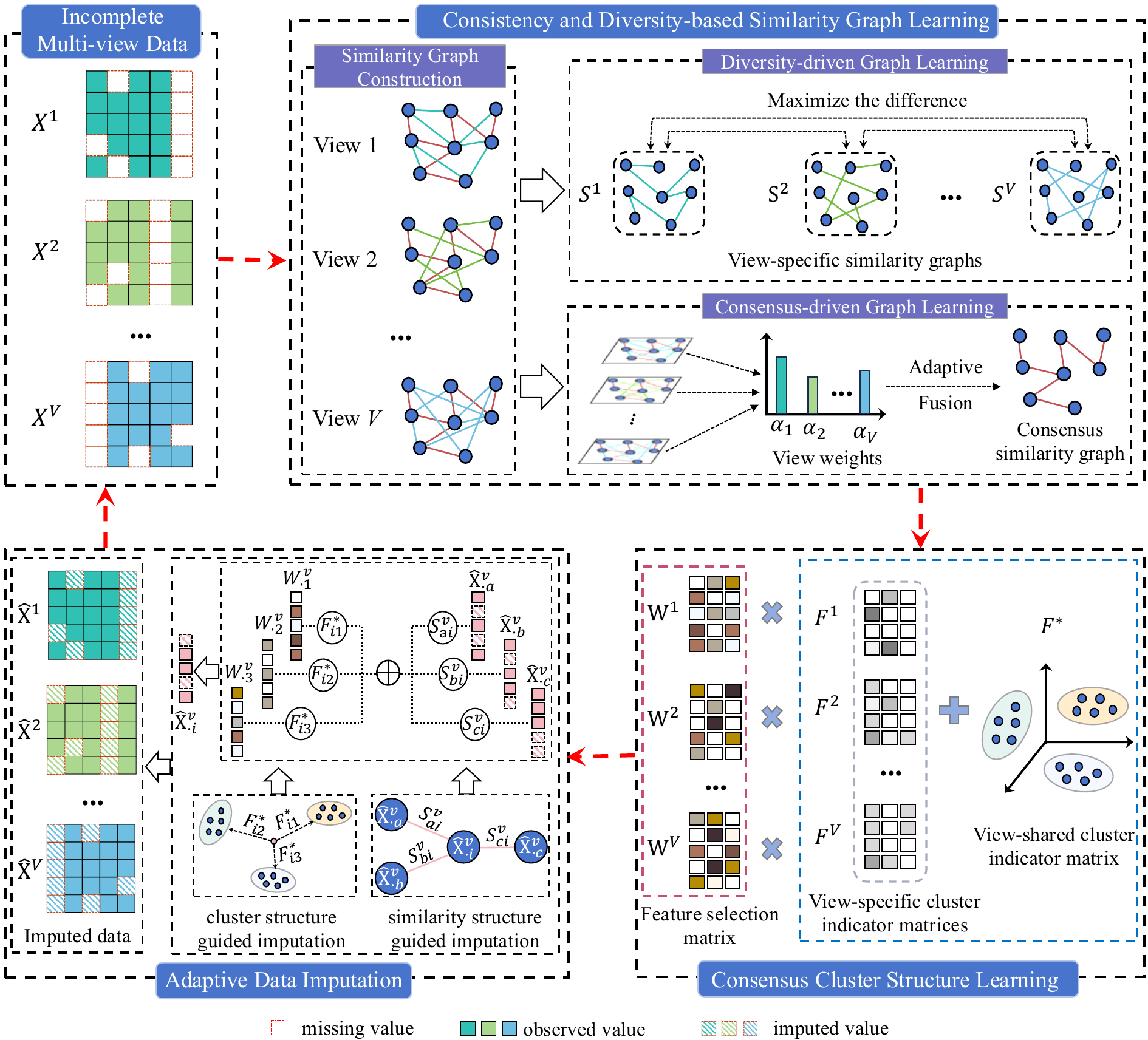} 
\caption{The framework of the proposed Cross-view joint Learning for mIxed-missing Multi-view unsupervised Feature Selection(CLIM-FS) method.}
\label{framework}
\end{figure}
\section{Proposed Method}\label{sec:proposed method}
\subsection{Notations} 
Throughout this paper, the bold uppercase letters (e.g. $\bm{A}$) denote matrices, while the italic lowercase letters (e.g. $a$) denote scalars. For any matrix $\bm{A}\in \mathbb{R}^{p \times q}$, we denote the $(i,j)$-th entry, the $i$-th row, and the $j$-th column of $\bm{A}$  by $\mathnormal{A}_{ij}$, $\bm{A}_{i \cdot}$, $\bm{A}_{\cdot j}$, respectively.  The Frobenius norm of $\bm{A}$ is defined as $\|\bm{A}\|_{\mathrm{F}}=\sqrt{\sum_{i=1}^{p}\sum_{j=1}^{q}\mathnormal{A}_{ij}^{2}}$. The $\ell_{2,1}$-norm and $\ell_{1}$-norm of $\bm{A}$ are given by $\|\bm{A}\|_{2,1}=\sum_{i=1}^{p}\sqrt{\sum_{j=1}^{q}\mathnormal{A}_{ij}^{2}}$ and $\|\bm{A}\|_{1}=\sum_{i=1}^{p}{\sum_{j=1}^{q}|\mathnormal{A}_{ij}|}$, respectively. $\operatorname{Tr}(\bm{A})$ and $\bm{A}^{\top}$  represent the trace and transpose of $\bm{A}$, respectively. The identity matrix is denoted by $\bm{I}$, and $\textbf{\textit{1}}=[1,\dots,1]^{\top}$ denotes a column vector of ones. 

Let $\mathcal{X}=\{\bm{X}^{v} \in \mathbb{R}^{d_v \times n}\}_{v=1}^{V}$ be an incomplete multi-view dataset with $V$ views, where $\bm{X}^{v}$  is the data matrix of the $v$-th view with $n$ samples and $d_v$ features. To describe the mixed-missing scenario involving both missing views and missing variables, we introduce a binary indicator matrix  $\bm{E}^{v} \in \{0,1\}^{d_v \times n}$ in $\bm{X}^{v}$, where $\mathnormal{E}^{v}_{ij}=0$ if $\mathnormal{X}^{v}_{ij}$ is missing, and $\mathnormal{E}^{v}_{ij}=1$ otherwise. For each multi-view sample $i$, if all variables are missing in a certain view (i.e., there exists $ v \in \{1,\dots, V\}$ such that $\bm{E}^{v}_{\cdot i} = \bm{0}$), it is referred to as a view-missing case. If only part of the variables within a view are missing (i.e., there exist $v \in \{1,\dots, V\}$ and $j \in \{1,\dots,d_{v}\}$ such that $\bm{E}^{v}_{ji} = 0$ while $\bm{E}^{v}_{\cdot i} \neq \bm{0}$), it is referred to as a variable-missing case. Such a mixed-missing scenario is frequently encountered in practical multi-view datasets. Our goal is to identify the most informative features from incomplete multi-view data under the mixed-missing scenario.

\subsection{Formulation of CLIM-FS}
Traditional approaches for selecting representative feature subsets from multi-view data in the mixed missing scenario generally begin by independently imputing the missing values in each view, followed by performing feature selection on the imputed data. This process can be summarized as follows:

\begin{enumerate}[label=(\roman*)]
\item \textit{Imputation stage}:
\begin{equation}
\begin{aligned}
\hat{\bm{X}}^{v} = \mathcal{Q}(\bm{X}^{v}), v=1,\dots,V
\end{aligned}
\end{equation}

\item \textit{Feature selection stage}:
\begin{equation}
\begin{aligned}
&\min_{\bm{W}^{v},\bm{F}^{v}}\sum_{v=1}^{V} \mathcal{L}(\hat{\bm{X}}^{v}, \bm{W}^{v}, \bm{F}^{v}) + \lambda \mathcal{R}(\bm{W}^{v}), \ \\
\end{aligned}
\end{equation} 
\end{enumerate}
where $\mathcal{Q}(\cdot)$ denotes the imputation operator (such as mean imputation), $\hat{\bm{X}}^{v} \in \mathbb{R}^{d_v \times n}$ represents the imputed data matrix of view $v$, $\mathcal{L}(\cdot)$ denotes the feature selection loss function, and $\bm{W}^{v} \in \mathbb{R}^{d_v \times c}$ and $\bm{F}^{v} \in \mathbb{R}^{n \times c}$ (with $c$ denoting the number of clusters)  represent the feature selection matrix and cluster indicator matrix for view $v$, respectively.$\mathcal{R}(\bm{W}^{v})$ denotes a sparse regularizer to eliminate less important features, and  $\lambda$ is the regularization parameter.
 
However, as previously discussed, this independent treatment of feature selection and data imputation overlooks the potential synergy between the two processes and fails to fully exploit inter-view associations, thereby limiting the feature selection performance. To address these issues, we propose integrating feature selection and data imputation within a joint learning framework that leverages cross-view diversity and consensus information to enhance collaborative learning. Specifically, we employ the non-negative orthogonal matrix factorization model~\cite{S.H.WangAAAI2015} to formulate the feature selection process, and impose $\ell_{2,1}$-norm regularization on the feature selection matrix $\bm{W}^{v}$ to promote the identification of discriminative features. Simultaneously, we introduce a general imputation constraint that explicitly incorporates both missing views and missing variables as optimization variables within the nonnegative orthogonal matrix factorization model. In this way, the missing values are jointly optimized in an alternating manner with other variables until convergence. The corresponding objective function is presented below.
\begin{equation}\label{3.2}
\begin{aligned}
&\min_{\hat{\bm{X}}^{v},\bm{W}^{v},\bm{F}^{v}} \sum_{v=1}^{V} \|\hat{\bm{X}}^{v} - \bm{W}^{v}\bm{F}^{v \top} \|_{\mathrm{F}}^{2} + \lambda \|\bm{W}^{v}\|_{2,1}\\
&\text { s.t. } \bm{E}^{v} \odot (\hat{\bm{X}}^{v} - \bm{X}^{v}) = \bm{0}, \bm{F}^{v} \geq \bm{0}, \bm{F}^{v \top}\bm{F}^{v} = \bm{I}.
\end{aligned}
\end{equation}
In Eq. (\ref{3.2}), feature selection and data imputation are seamlessly integrated into a unified learning framework instead of being handled independently. Moreover, whether a sample has all or only some variables missing in a view, its missing values can be adaptively imputed during the feature selection process, while the observed values remain unchanged due to the constraint  $\bm{E}^{v} \odot (\hat{\bm{X}}^{v} - \bm{X}^{v}) = \bm{0}$. Therefore, unlike conventional IMUFS methods, which are limited to addressing the view-missing problem, the proposed method can handle the more general mixed-missing scenario involving both missing views and missing variables.

Furthermore, to enhance the collaborative learning between feature selection and data imputation by exploiting both cross-view diversity and consensus information,  we first decompose the clustering indicator matrix of each view into a shared component and a view-specific component. The shared component captures the consensus clustering structure across multiple views, while the view-specific component reflects label noise unique to each view. We further impose spectral graph regularization on the shared component to ensure that similar data points receive similar clustering labels. Meanwhile, we regularize the view-specific component using the $\ell_1$-norm to enforce sparsity and eliminate label inconsistencies caused by view noise. As a result, we reformulate Eq. (\ref{3.2}) as follows:

\begin{equation}\label{3.3}
\begin{aligned}
\min_{\substack{\hat{\bm{X}}^{v}\!,\bm{W}^{v}\\ \bm{F}^{v}\!, \bm{F}^{*}}}& \sum_{v=1}^{V} \|\hat{\bm{X}}^{v} - \bm{W}^{v}(\bm{F}^{*} +\bm{F}^{v} )^{\top} \|_{\mathrm{F}}^{2} + \lambda \|\bm{W}^{v}\|_{2,1}\\
&+ \beta \| \bm{F}^{v} \|_{1} + \operatorname{Tr}( \bm{F}^{* \top}\bm{L}_{H}\bm{F}^{*} )\\
\text { s.t. } \bm{E}&^{v} \odot (\hat{\bm{X}}^{v} - \bm{X}^{v}) = \bm{0}, \bm{F}^{*} \geq \bm{0}, \bm{F}^{*\top}\bm{F}^{*} = \bm{I},
\end{aligned}
\end{equation}
where $\bm{F}^{*}$ and $\bm{F}^{v}$ denote the view-shared and view-specific cluster indicator matrices, respectively. Additionally, $\bm{L}_{H} = \bm{D}_{H}-\bm{H} \in \mathbb{R}^{n \times n}$ represents the Laplacian matrix, where $\bm{H}$ denotes the consensus similarity matrix across different views, and the degree matrix $\bm{D}_{H}$ is a diagonal, with its $i$-th diagonal entry given by $\sum_{j=1}^{n}H_{ji}$. The details of how $\bm{H}$ is learned will be provided later. In Eq.~(\ref{3.3}), missing data are imputed using the consensus cluster structure information across different views. Moreover, Eq.~(\ref{3.3}) ensures that samples within the same cluster remain close after imputation, while those from different clusters stay well separated. Thus, the proposed method preserves both intra-cluster and inter-cluster relationships in the imputed data, thereby providing reliable discriminative information to enhance feature selection performance. These properties are theoretically established in Theorem 1 of Section 5.1.

Previous studies have shown that constructing nearest neighbor graphs to preserve local geometric structure is crucial for improving the performance of unsupervised feature selection~\cite{X.W.LiuTNNLS2013,B.L.ChenTPAMI2022}. To more effectively learn similarity matrices that maintain the local geometric structure in multi-view data, we propose a novel cross-view similarity graph learning method that simultaneously captures both the consistency and diversity of local geometric structures across multiple views. Specifically, we adaptively learn a similarity-induced graph for each view and capture cross-view diversity by maximizing the discrepancy between similarity matrices from different views. Additionally, our approach explores the intrinsic consistency of local geometric structures by fusing the similarity matrices across views using adaptive view weights. Formally, the learning objective of our method is expressed as follows:
\begin{equation}\label{3.4}
\begin{aligned}
\min_{\bm{S}^{v}\!,\bm{H}\!,\bm{\alpha}} &\sum_{v=1}^{V} \big[ \frac{1}{2}\!\sum_{i,j=1}^{n} \|\hat{\bm{X}}_{\cdot i}^{v}-\hat{\bm{X}}_{\cdot j}^{v}\|_{2}^{2}{S}^{v}_{ij} +\!\! \sum_{m =1}^{V}\!\alpha_{v}\alpha_{m}\!\operatorname{Tr}(\bm{S}^{v}\bm{S}^{m\top} \!) \\
&+ \xi_{v}\|\bm{S}^{v}\|_{\mathrm{F}}^{2}\big] -\big[ \operatorname{Tr}(\bm{H}^{\top}\sum_{v=1}^{V}\alpha_{v}\bm{S}^{v}) - \gamma\|\bm{H}\|_{\mathrm{F}}^{2}\big] \\
\text { s.t. }\mathnormal{S}&^{v}_{ij}\geq 0, \textbf{\textit{1}}^{\top}\bm{S}_{\cdot i}^{v}=1,\|\bm{S}_{\cdot i}^{v}\|_{0}={k},\mathnormal{H}_{ij}\geq 0,\textbf{\textit{1}}^{\top}\bm{H}_{\cdot i}=1,\\\|&\bm{H}_{\cdot i}\|_{0}={k}, \bm{\alpha}^{\top}\bm{1} = 1,\alpha_{v} \geq 0,
\end{aligned}
\end{equation}
where $\bm{\alpha}=[\alpha_{1},\dots,\alpha_{v}]^{\top}$ denotes the view weight vector, $\gamma$ is the regularization parameter, and $\bm{H}\in \mathbb{R}^{n \times n}$ and $\bm{S}^{v}\in \mathbb{R}^{n \times n}$ represent the consensus and view-specific similarity matrix, respectively. In addition, the constraints $\|\bm{S}_{\cdot i}^{v}\|_{0}={k}$ and $\|\bm{H}_{\cdot i}\|_{0}={k}$ ensure that each data point is connected only to its $k$ nearest neighbors in the similarity-induced graphs. As shown in Eq.~(\ref{3.4}), the proposed method leverages both intra-view and inter-view neighborhood relationships to guide data imputation. Furthermore, it ensures that similar sample pairs within and across views remain similar after imputation, thereby facilitating the preservation of local geometric structure. This, in turn, effectively enhances feature selection performance under the incomplete multi-view scenario. Theoretical guarantees for these properties are provided in Theorem 2 of Section 5.1.

By combining Eq.~(\ref{3.3}) and Eq.~(\ref{3.4}), we obtain the final objective function of the proposed CLIM-FS as follows:
\begin{equation}\label{3.5}
\begin{aligned}
\min_{\bm{\mathit{{\Theta}}}}& \!\sum_{v=1}^{V} \!\big[ \|\hat{\bm{X}}^{v} \!-\! \bm{W}^{v}(\bm{F}^{v} + \bm{F}^{*})^{\top} \|_{\mathrm{F}}^{2} + \lambda \|\bm{W}^{v}\|_{2,1} + \beta\| \bm{F}^{v} \|_{1}\\
+\frac{1}{2}&\!\!\sum_{i,j=1}^{n} \!\!\|\hat{\bm{X}}_{\cdot i}^{v}\!-\!\hat{\bm{X}}_{\cdot j}^{v}\|_{2}^{2}{S}^{v}_{ij} \! + \!\!\sum_{m =1}^{V}\!\!\alpha_{v}\alpha_{m}\!\operatorname{Tr}(\bm{S}^{v}\bm{S}^{m\top} ) \!+ \xi_{v}\|\bm{S}^{v}\|_{\mathrm{F}}^{2} \big]\\
-\big[ &\operatorname{Tr}(\bm{H}^{\top}\sum_{v=1}^{V}\alpha_{v}\bm{S}^{v}) - \gamma\|\bm{H}\|_{\mathrm{F}}^{2}\big] + \operatorname{Tr}( \bm{F}^{*\top}\bm{L}_{{H}}\bm{F}^{*} )\\
\text { s.t. }&\bm{E}^{v} \!\odot (\hat{\bm{X}}^{v} \!-\! \bm{X}^{v}) = \bm{0}, \bm{F}^{*} \!\geq\! \bm{0}, \bm{F}^{*\top}\!\bm{F}^{*} \!=\!\bm{I},\textbf{\textit{1}}^{\top}\bm{S}_{\cdot i}^{v}=1\\
&\mathnormal{S}^{v}_{ij}\geq 0,\|\bm{S}_{\cdot i}^{v}\|_{0}={k},\textbf{\textit{1}}^{\top}\bm{H}_{\cdot i}=1,\mathnormal{H}_{ij}\geq 0,\|\bm{H}_{\cdot i}\|_{0}\!=\!{k}\\
&\bm{\alpha}^{\top}\bm{1} = 1, \alpha_{v} \geq 0,
\end{aligned}
\end{equation}
where $\bm{\mathit{{\Theta}}}\!=\!\{\{\hat{\bm{X}}^{v}\}_{v=1}^{V}, \{\bm{W}^{v}\}_{v=1}^{V}, \{\bm{F}^{v}\}_{v=1}^{V},\{\bm{S}^{v}\}_{v=1}^{V},\bm{F}^{*},$ $\bm{H},\bm{\alpha}\}$. In Eq.~(\ref{3.5}), missing data are effectively imputed with guidance from the consensus cluster structure and cross-view local geometric information. This refined imputation subsequently enhances the structural learning of multi-view data, thereby improving feature selection performance. Moreover, although Eq. (3.5) includes four parameters, namely $\lambda$, $\beta$, $\xi_v$, and $\gamma$, the regularization parameters $\xi_v$ and $\gamma$ are automatically determined during optimization  of $\bm{S}^{v}$ and $\bm{H}$, respectively, as detailed in the next section. Therefore, only $\lambda$ and $\beta$ require manual tuning in the proposed CLIM-FS.

In summary, the proposed CLIM-FS in Eq.~(\ref{3.5}) offers two key advantages: (\romannumeral1) Compared with existing IMUFS methods, which are limited to handling only the view-missing issue, CLIM-FS can address the more general mixed-missing problem, as well as its two special cases, namely the view-missing and variable-missing issues, thereby offering broader applicability in real-world scenarios. (\romannumeral2) Instead of relying exclusively on intra-view information for joint feature selection and data imputation, CLIM-FS fully leverages both the diversity and consensus across different views to enhance the collaborative interaction between feature selection and data imputation, thereby improving feature selection performance. 

\section{Optimization}\label{sec:optimization}
As shown in Eq.~(\ref{3.5}), the objective function is not convex with respect to all of its variables. Therefore, we develop an alternating iterative algorithm to solve this optimization problem, in which each variable is optimized in turn while the others are held fixed.

\subsection{Update $\bm{W}^{v}$ by Fixing Other Variables}
With the other variables fixed, $\bm{W}^{v}$ can be updated by solving the following optimization problem:
\begin{equation}\label{4.5}
\begin{aligned}
\min_{\bm{W}^{v}} & \|\hat{\bm{X}}^{v} - \bm{W}^{v}(\bm{F}^{v} + \bm{F}^{*})^{\top} \|_{\mathrm{F}}^{2} + \lambda \|\bm{W}^{v}\|_{2,1}.\\
\end{aligned}
\end{equation}

According to~\cite{X.L.LiTNNLS2018}, problem (\ref{4.5}) can be equivalently transformed into the following form:
\begin{equation}\label{4.6}
\begin{aligned}
\min_{\bm{W}^{v}} & \|\hat{\bm{X}}^{v} - \bm{W}^{v}(\bm{F}^{v} + \bm{F}^{*})^{\top} \|_{\mathrm{F}}^{2} + \lambda\operatorname{Tr}(\bm{W}^{v\top}\bm{D}^{v}\bm{W}^{v}),\\
\end{aligned}
\end{equation}
where $\mathnormal{D}^{v}_{ii} = 1 / 2\sqrt{\|\bm{W}_{i \cdot}^{v}\|_{2}^{2}+\epsilon}$, and $\epsilon$ is a small constant to prevent the denominator from vanishing.

By taking the derivative of Eq. (\ref{4.6}) w.r.t. to $\bm{W}^{v}$ and setting it to zero, we obtain
\begin{equation}\label{4.7}
\lambda\bm{D}^{v}\bm{W}^{v} + \bm{W}^{v}(\bm{F}^{v}+\bm{F}^{*})^{\top}(\bm{F}^{v}+\bm{F}^{*})=\hat{\bm{X}}^{v}(\bm{F}^{v}+\bm{F}^{*}). 
\end{equation}
Problem (\ref{4.7}) is a standard form of the Sylvester matrix equation, which can be effectively solved using the Bartels-Stewart algorithm~\cite{R.A. Horn2012}.

\subsection{Update $\bm{F}^{v}$ by Fixing Other Variables}
By fixing other variables, the optimization problem w.r.t $\bm{F}^{v}$ can be transformed into
\begin{equation}\label{4.8}
\begin{aligned}
\min_{\bm{F}^{v}} \sum_{v=1}^{V} \|\hat{\bm{X}}^{v} - \bm{W}^{v}(\bm{F}^{v} + \bm{F}^{*})^{\top} \|_{\mathrm{F}}^{2} + \beta\| \bm{F}^{v} \|_{1}.
\end{aligned}
\end{equation}

Problem (\ref{4.8}) can be efficiently solved using the proximal gradient descent method~\cite{A.NitandaNIPS2014}, as detailed below:

\begin{equation}\label{4.9}
\begin{aligned}
\mathnormal{F}^{v}_{ij} = \left[\Omega( \bm{F}^{v} - t\nabla(\bm{F}^{v}))\right]_{ij},
\end{aligned}
\end{equation}
where $\nabla(\bm{F}^{v})\!=\!-2\hat{\bm{X}}^{v\top}\bm{W}^{v}\!+\!2\bm{F}^{v}\bm{W}^{v\top}\bm{W}^{v}\!+\!2\bm{F}^{*}\bm{W}^{v\top}\bm{W}^{v}$, and $t$ is the step size. $\Omega(\cdot)$ is a soft thresholding operator, which is defined as follows:
\begin{equation}\label{4.10}
\begin{aligned}
\left[\Omega(\bm{A})\right]_{ij} = {\mathnormal{sign}}(\mathnormal{A}_{ij})\mathnormal{max}(|\mathnormal{A}_{ij}|-\beta,0),
\end{aligned}
\end{equation}
where $\mathnormal{sign}(\cdot)$ is the sign function. In addition, to accelerate the convergence of Eq. (\ref{4.9}),  we employ the Adam algorithm~\cite{ADAM} to adaptively adjust the step size $t$.

\subsection{Update $\bm{F}^{*}$ by Fixing Other Variables}
When other variables are fixed, the optimization problem w.r.t. $\bm{F}^{*}$ can be reduced to:
\begin{equation}\label{4.11}
\begin{aligned}
\min_{\bm{F}^{*}} &\!\sum_{v=1}^{V} \|\hat{\bm{X}}^{v} - \bm{W}^{v}(\bm{F}^{v} + \bm{F}^{*})^{\top} \|_{\mathrm{F}}^{2} + \operatorname{Tr}( \bm{F}^{*\top}\bm{L}_{{H}}\bm{F}^{*} )\\
\text { s.t. } &\bm{F}^{*} \geq \bm{0}, \bm{F}^{*\top}\bm{F}^{*} = \bm{I}.
\end{aligned}
\end{equation}

The Lagrangian function for problem (\ref{4.11}) is given as follows:
\begin{equation}\label{4.12}
\begin{aligned}
&\mathcal{L}(\bm{F}^{*},\bm{\phi})=\sum_{v=1}^{V}\operatorname{Tr}(-2\hat{\bm{X}}^{v\top}\bm{W}^{v}\bm{F}^{*\top}+\bm{F}^{v}\bm{W}^{v\top}\bm{W}^{v}\bm{F}^{*\top}+\\
&\bm{F}^{*}\bm{W}^{v\top}\bm{W}^{v}\bm{F}^{v\top}+\bm{F}^{*}\bm{W}^{v\top}\bm{W}^{v}\bm{F}^{*\top}) + \rho\|\bm{F}^{*\top}\bm{F}^{*}-\bm{I}\|_{\mathrm{F}}^{2}\\
&-\operatorname{Tr}(\bm{\phi}^{\top}\bm{F}^{*})+\operatorname{Tr}(\bm{F}^{*\top}\bm{L}_{{H}}\bm{F}^{*}),
\end{aligned}
\end{equation}
where $\bm{\phi}$ is the Lagrange multiplier, and $\rho$ is a large constant to guarantee the orthogonality of $\bm{F}^{*}$.

Taking the derivative of $\mathcal{L}(\bm{F}^{*},\bm{\phi})$ w.r.t. $\bm{F}^{*}$ and utilizing the Karush-Kuhn-Tucker (KKT) complementary condition $\bm{\phi}_{ij}\bm{F}^{*}_{ij}=0$~\cite{KKT}, we derive the update rule for $\bm{F}^{*}$ as follows:
\begin{equation}\label{4.14}
\begin{aligned}
\bm{F}^{*}\!\!\!=\!\bm{F}^{*} \!\!\odot\! \frac{\sum_{v=1}^{V}(\bm{J}_{[+]}^{v}\!+\!\bm{M}_{[-]}^{v}\!+\!\bm{F}^{*}\bm{U}_{[-]}^{v})+\bm{H}\bm{F}^{*}+2\rho\bm{F}^{*}}{\sum_{v=1}^{V}(\bm{J}_{[-]}^{v}\!+\!\bm{M}_{[+]}^{v}\!+\!\bm{F}^{*}\bm{U}_{[+]}^{v})+\bm{D}_{{H}}\bm{F}^{*}+2\rho\bm{K}},
\end{aligned}
\end{equation}
where  $\odot$ denotes the element-wise product, $\bm{J}^{v}=\hat{\bm{X}}^{v\top}\bm{W}^{v}$, $\bm{U}^{v}=\bm{W}^{v\top}\bm{W}^{v}$, $\bm{M}^{v}=\bm{F}^{v}\bm{U}^{v}$, and $\bm{K}=\bm{F}^{*}\bm{F}^{*\top}\bm{F}^{*}$. Furthermore, for any matrix $\bm{A}$, $\bm{A}_{[-]}$ and $\bm{A}_{[+]}$ are defined as $\bm{A}_{[-]}=1/2(|\bm{A}|-\bm{A})$ and $\bm{A}_{[+]}=1/2(|\bm{A}|+\bm{A})$, respectively.

\subsection{Update $\bm{S}^{v}$ by Fixing Other Variables}
With other variables fixed, each column of $\bm{S}^{(v)}$ is independent of the others,  which allows  $\bm{S}^{(v)}$ to be optimized by solving each column $\bm{S}_{\cdot j}^{v}$ separately, as follows:
\begin{equation}\label{4.17}
\begin{aligned}
\min_{\bm{S}_{\cdot j}^{v}}&\sum_{i=1}^{n}(\frac{1}{2}\|\hat{\bm{X}}_{\cdot i}^{v}-\hat{\bm{X}}_{\cdot j}^{v}\|_{2}^{2}-\alpha_{v}\mathnormal{H}_{ij}+\sum_{m \neq v}^{V}\alpha_{v}\alpha_{m}\mathnormal{S}_{ij}^{m})\mathnormal{S}_{ij}^{v}+\\
&(\xi_{v}+\alpha_{v}^{2})\|\bm{S}_{\cdot j}^{v}\|_{2}^{2}\\
\text { s.t. } &\mathnormal{S}^{v}_{ij}\geq 0, \textbf{\textit{1}}^{\top}\bm{S}_{\cdot j}^{v}=1, \|\bm{S}_{\cdot j}^{v}\|_{0}={k}.
\end{aligned}
\end{equation}

By defining a column vector $\bm{q}_{j}^{v} \in \mathbb{R}^{n \times 1}$, whose $i$-th entry is given by $q_{ij}^{v}=\frac{1}{2}\|\hat{\bm{X}}_{\cdot i}^{v}-\hat{\bm{X}}_{\cdot j}^{v}\|_{2}^{2}-\alpha_{v}\mathnormal{H}_{ij}\!+\!\sum_{m \neq v}^{V}\!\alpha_{v}\alpha_{m}\mathnormal{S}_{ij}^{m}$, the problem (\ref{4.17}) can be reformulated as follows:
\begin{equation}\label{4.18}
\begin{aligned}
&\min _{\bm{S}_{\cdot j}^{v}} \frac{1}{2}\left\|\bm{S}_{\cdot j}^{v}+\bm{q}_{j}^{v}/2 (\xi_{v}+\alpha_{v}^{2})\right\|_{2}^{2} \\
\text { s.t. } &\mathnormal{S}^{v}_{ij}\geq 0, \textbf{\textit{1}}^{\top}\bm{S}_{\cdot j}^{v}=1, \|\bm{S}_{\cdot j}^{v}\|_{0}={k}.
\end{aligned}
\end{equation}

The Lagrangian function of the problem (\ref{4.18}) can be written as follows:
\begin{equation}\label{add1}
\begin{aligned}
\mathcal{L}(\bm{S}_{\cdot j}^{v},\psi,\bm{\varphi}) =&  \frac{1}{2}\left\|\bm{S}_{\cdot j}^{v}+\bm{q}_{j}^{v}/2 (\xi_{v}+\alpha_{v}^{2})\right\|_{2}^{2} - \psi(\bm{1}^{T} \bm{S}_{\cdot j}^{v}-1)\\
&-\bm{\varphi}^{T}\bm{S}_{\cdot j}^{v},
\end{aligned}
\end{equation}
where $\psi$ and $\bm{\varphi}^{T}$ are the Lagrangian multipliers.

By taking the derivative of $\mathcal{L}(\bm{S}_{\cdot j}^{v},\psi,\bm{\varphi})$ w.r.t $\bm{S}_{\cdot j}^{v}$ and setting  it to zero, we have
\begin{equation}\label{add2}
\begin{aligned}
\bm{S}_{\cdot j}^{v} + \bm{q}_{j}^{v}/2 (\xi_{v}+\alpha_{v}^{2}) - \psi \bm{1}^{T} - \bm{\varphi}^{T}=  \bm{0}.
\end{aligned}
\end{equation}
 
According to the KKT complementary condition, i.e., $\mathnormal{S}_{ij}^{v}\varphi_{i}=0$, we can obtain the solution for $\bm{S}_{\cdot j}^{v}$ as follows:
\begin{equation}
    \begin{aligned}
\hat{\mathnormal{S}}_{ij}^{v}=\mathnormal{max}(\psi - {q_{ij}^{v}}/{2(\xi_{v}+\alpha_{v}^{2})},0). 
\end{aligned}
\end{equation}

To ensure that the constraint $\|\bm{S}_{\cdot j}^{v}\|_{0}={k}$ is satisfied, we first arrange $q_{1j}^{v} \dots q_{nj}^{v} $ in ascending order, and then set $\hat{\mathnormal{S}}_{kj}^{v}>0$ and $\hat{\mathnormal{S}}_{k+1,j}^{v}=0$ so that $\bm{S}_{\cdot j}^{v}$ has only $k$ nonzero entries. Hence, we have 
\begin{equation}
    \begin{aligned}
\psi - {q_{kj}^{v}}/{2(\xi_{v}+\alpha_{v}^{2})}>0 , \psi - {q_{k+1,j}^{v}}/{2(\xi_{v}+\alpha_{v}^{2})} \leq 0. 
\end{aligned}
\end{equation}

Together with the constraint $\bm{1}^{T} \bm{S}_{\cdot j}^{v}=1$, we obtain $\psi = \frac{1}{k}+\frac{1}{2k(\xi_{v}+\alpha_{v}^{2})}\sum_{t=1}^{k}q_{tj}^{v}$. By setting $\xi_{v}=(k q_{k+1,j}^{v}-\sum_{t=1}^{k} q_{tj}^{v})/2 - \alpha_{v}^{2}$, we have the optimal solution of $\mathnormal{S}_{ij}^{v}$ as follows:
\begin{equation}\label{4.18.5}
\mathnormal{S}_{i j}^{v}=\left\{\begin{array}{cc}
\frac{q_{k+1,j}^{v}-q_{i j}^{v}}{k q_{k+1,j}^{v}-\sum_{t=1}^{k} q_{t j}^{v}} & j \leq k; \\
0 & j>k.
\end{array}\right.
\end{equation}

\begin{algorithm}[t]
\caption{Iterative Algorithm of CLIM-FS}
\KwIn{\begin{enumerate}
\item Incomplete multi-view data $\{\bm{X}^{(v)}\!\in\! \mathbb{R}^{d_v \times n}\}_{v=1}^{V}$;
\item the parameters $\beta$ and $\lambda$;
\end{enumerate}
}

\textbf{Initialize:} $\{\bm{W}^{v}\}_{v=1}^{V}$, $\{\bm{S}^{v}\}_{v=1}^{V}$, $\bm{H}$, $\{\bm{F}^{v}\}_{v=1}^{V}$, $\bm{F}^{*}$,  $\bm{\alpha}$.

\While{not convergent}{
\quad 1. Update $\{\bm{W}^{v}\}_{v=1}^{V}$ by solving Eq. (\ref{4.7});\\
\quad 2. Update $\{\mathnormal{D}^{v}_{ii} = 1 / 2\sqrt{\|\bm{W}_{i \cdot}^{v}\|_{2}^{2}+\epsilon}\}_{v=1}^{V}$;\\
\quad 3. Update $\{\bm{F}^{v}\}_{v=1}^{V}$ via Eq. (\ref{4.9});\\
\quad 4. Update $\bm{F}^{*}$ via Eq. (\ref{4.14});\\
\quad 5. Update $\{\bm{S}^{v}\}_{v=1}^{V}$ via Eq. (\ref{4.18.5});\\
\quad 6. Update $\bm{H}$ via Eq. (\ref{4.23});\\
\quad 7. Update $\bm{\alpha}$ by solving problem (\ref{4.25});\\
\quad 8. Update $\{\hat{\bm{X}}^{v}\}_{v=1}^{V}$ via Eq. (\ref{4.4}).\\
}

\KwOut{Sorting the $\ell_{2}$-norm of rows of $\{\bm{W}^{v} \}_{v=1}^{V}$ in descending order and selecting the top $r$ features from $\hat{\bm{X}}^{v}$.}
\end{algorithm}

\subsection{Update $\bm{H}$ by Fixing Other Variables}
When other variables are fixed, the objective function w.r.t $\bm{H}$ can be rewritten as follows:
\begin{equation}\label{4.19}
\begin{aligned}
&\min_{\bm{H}} -\operatorname{Tr}(\bm{H}^{\top}\sum_{v=1}^{V}\alpha_{v}\bm{S}^{v}) + \gamma\|\bm{H}\|_{\mathrm{F}}^{2} + \operatorname{Tr}( \bm{F}^{*\top}\bm{L}_{\mathrm{H}}\bm{F}^{*} )\\
&\text { s.t. }\mathnormal{H}_{ij}\geq 0, \textbf{\textit{1}}^{\top}\bm{H}_{\cdot i}=1,\|\bm{H}_{\cdot i}\|_{0}={k}.
\end{aligned}
\end{equation}

According to the matrix trace property, we can transform problem (\ref{4.19}) into the following element-wise form:
\begin{equation}\label{4.20}
\begin{aligned}
&\min_{\bm{H}}\!\sum_{i,j=1}^{n}\frac{1}{2}\|\bm{F}^{*}_{i \cdot}-\bm{F}^{*}_{j \cdot}\|_{2}^{2}H_{ij} + \gamma H_{ij}^{2}-H_{ij}P_{ij}\\
&\text { s.t. } \mathnormal{H}_{ij}\geq 0, \textbf{\textit{1}}^{\top}\bm{H}_{\cdot j}=1, \|\bm{H}_{\cdot j}\|_{0}={k}.
\end{aligned}
\end{equation}
where $\bm{P} = \sum_{v=1}^{V}\alpha_{v}\bm{S}^{v}$. It can be seen that problem (\ref{4.20}) is independent for different columns. Thus, we can follow the same procedure used in optimizing $\bm{S}^{v}$ to obtain the optimal solution for $\bm{H}$ below:

\begin{equation}\label{4.23}
\mathnormal{H}_{i j}=\left\{\begin{array}{cc}
\frac{b_{k+1,j}-b_{i j}}{k b_{k+1,j}-\sum_{t=1}^{k} b_{t j}} & j \leq k; \\
0 & j>k.
\end{array}\right.
\end{equation}
where $\bm{b}_{j} \in \mathbb{R}^{n \times 1}$ is a column vector with the $i$-th entry defined as $b_{ij}=\frac{1}{2}\|\bm{F}^{*}_{i \cdot}-\bm{F}^{*}_{j \cdot}\|_{2}^{2}-\mathnormal{P}_{ij}$, and $\gamma$ is set to $(k b_{k+1,j}-\sum_{t=1}^{k} b_{t j})/2$  to guarantee that $\|\bm{H}_{\cdot j}\|_{0}={k}$ holds.

\subsection{Update $\bm{\alpha}$ by Fixing Other Variables}
While fixing other variables, the corresponding objective function w.r.t $\bm{\alpha}$ becomes
\begin{equation}\label{4.24}
\begin{aligned}
&\min_{\bm{\alpha}} \!\!\sum_{v,m=1}^{V}\!\!\alpha_{v}\alpha_{m}\!\operatorname{Tr}(\bm{S}^{v}\!\bm{S}^{m\top} \!) - \operatorname{Tr}(\bm{H}^{\top}\!\sum_{v=1}^{V}\!\alpha_{v}\bm{S}^{v})\\
&\text { s.t. }\bm{\alpha}^{\top}\bm{1} = 1,\alpha_{v} \geq 0.
\end{aligned}
\end{equation}

Problem (\ref{4.24}) can be equivalently reformulated as follows:
\begin{equation}\label{4.25}
\begin{aligned}
\min_{\bm{\alpha}}& \sum_{v=1}^{V}\alpha_{v}^{2}\operatorname{Tr}(\bm{S}^{v}\bm{S}^{v\top} )+\sum_{v=1}^{V}\sum_{m \neq v}^{V}\alpha_{v}\alpha_{m}\operatorname{Tr}(\bm{S}^{v}\bm{S}^{m\top} )\\
&-\sum_{v=1}^{V}\alpha_{v}\operatorname{Tr}(\bm{H}^{\top}\bm{S}^{v})\\
\text { s.t. }&\bm{\alpha}^{\top}\bm{1} = 1,\alpha_{v} \geq 0.
\end{aligned}
\end{equation}
Problem (\ref{4.25}) is a standard quadratic objective function subject to linear constraints. It can be efficiently solved using the active-set method~\cite{E.Wong2011}.

\subsection{Update $\hat{\bm{X}}^{v}$ by Fixing Other Variables}
After fixing the other variables, the objective function w.r.t. $\hat{\bm{X}}^{v}$ becomes
\begin{equation}\label{4.1}
\begin{aligned}
\min_{\hat{\bm{X}}^{v}} & \|\hat{\bm{X}}^{v} - \bm{W}^{v}(\bm{F}^{v} + \bm{F}^{*})^{\top} \|_{\mathrm{F}}^{2}+\operatorname{Tr}(\hat{\bm{X}}^{v}\bm{L}^{v}\hat{\bm{X}}^{v\top})\\
\text { s.t. } &\bm{E}^{v} \odot (\hat{\bm{X}}^{v} - \bm{X}^{v}) = \bm{0}.
\end{aligned}
\end{equation}

By taking the derivative of Eq. (\ref{4.1}) w.r.t. $\hat{\bm{X}}^{v}$ and setting it to zero, we can obtain the optimal solution of $\hat{\bm{X}}^{v}$ as follows:
\begin{equation}\label{4.4}
\hat{\bm{X}}^{v} = \bm{R}^{v} + \bm{E}^{v} \odot (\bm{X}^{v} - \bm{R}^{v}),
\end{equation}
where $\bm{R}^{v} = \bm{W}^{v}(\bm{F}^{v}+\bm{F}^{*})^{\top}(\bm{I} + \bm{L}^{v})^{-1}$.

Algorithm 1 summarizes the optimization process of the proposed CLIM-FS. In the initialization step, $\alpha_{v}$ is set to ${1}/{V}$ for all views, $\bm{W}^{(v)}$ is initialized as a matrix with all entries equal to 1, $\bm{S}^{(v)}$ and $\bm{H}$ are initialized as $k$-nearest neighbor graph following~\cite{X.L.LiTIP2019}, $\bm{F}^{(v)}$ and $\bm{F}^{*}$ are initialized via spectral clustering~\cite{A.NgNIPS2001}, and the missing values in $\hat{\bm{X}}^{(v)}$ are initialized using the mean values of the observed entries in $\bm{X}^{v}$.

\section{Model Analysis}\label{sec:Theoretical Analysis}

\subsection{Analysis of the Collaborative Learning Mechanism of CLIM-FS}\label{sec:Collaborative Learning Mechanism}

In this section, we theoretically analyze how the proposed CLIM-FS method leverages the consensus cluster structure and the cross-view local geometric structure to enhance collaborative learning between feature selection and data imputation. Specifically, according to the update rule in Eq.~(\ref{4.4}), the imputation for sample $\bm{X}_{\cdot i}^{v}$ is given by
\begin{equation}\label{C1}
\hat{\bm{X}}_{\cdot i}^{v} = \frac{1}{1+\sum_{j=1}^{n}S_{ji}^{v}}(\bm{W}^{v}(\bm{F}_{i \cdot}^{v}+\bm{F}_{i \cdot}^{*})^{\top}+\sum_{j=1}^{n}S_{ji}^{v}\hat{\bm{X}}_{\cdot j}^{v}),
\end{equation}
In addition, the constraint $\bm{E}^{v} \odot (\hat{\bm{X}}^{v} - \bm{X}^{v}) = \bm{0}$ in Eq.~(\ref{3.5}) ensures that this imputation only affects the missing entries in $\bm{X}_{\cdot i}^{v}$, while the observed entries remain unchanged. Given the constraints $\|\bm{S}_{\cdot i}^{v}\|_{0}={k}$ and $\textbf{\textit{1}}^{\top}\bm{S}_{\cdot i}^{v}=1$ in Eq.~(\ref{3.5}), we can rewrite Eq.~(\ref{C1}) as follows:
\begin{equation}\label{C2}
\hat{\bm{X}}_{\cdot i}^{v} = \frac{1}{2}\bm{W}^{v}(\bm{F}_{i \cdot}^{v}+\bm{F}_{i \cdot}^{*})^{\top}+\frac{1}{2}\!{\sum_{j\in \mathcal{N}^{v}(i)}\!\!\!S_{ji}^{v}\hat{\bm{X}}_{\cdot j}^{v}},
\end{equation}
where $\mathcal{N}^{v}(i)$ denotes the $k$ nearest neighbors of sample $\bm{X}_{\cdot i}^{v}$. 

To provide a clearer understanding of the imputation process described in Eq.~(\ref{C2}), which is jointly guided by the feature selection results, the cluster structure, and the local geometric structure, we begin by analyzing the first term $\bm{C}_{\cdot i}^{v} \triangleq \bm{W}^{v}(\bm{F}_{i \cdot}^{v}+\bm{F}_{i \cdot}^{*})^{\top}$ in Eq.~(\ref{C2}). This analysis demonstrates how the feature selection results and consensus cluster structure contribute to the data imputation. Since sparsity is imposed on $\bm{F}^{v}$ during the optimization process, $\bm{C}_{\cdot i}^{v}$ can be approximated as follows:
\begin{equation}\label{C3}
\bm{C}_{\cdot i}^{v} \approx \bm{W}^{v}\bm{F}_{i \cdot}^{* \top}=\sum_{p=1}^{c}F_{ip}^{*}\bm{W}_{\cdot p}^{v}.
\end{equation}
From the perspective of matrix factorization-based clustering, $\bm{W}_{\cdot p}^{v}$ serves as the prototype for the $p$-th cluster, while $F_{ip}^{*}$ indicates the probability that the sample $\bm{X}_{\cdot i}^{v}$ is assigned to the $p$-th cluster~\cite{Tang.CKBS2018}. Thus, Eq.~(\ref{C3}) reveals that the proposed CLIM-FS informs data imputation by linearly combining cluster prototypes, weighted by sample-to-cluster assignment probabilities.  Moreover, the $\ell_{2,1}$-norm regularization $\|\bm{W}^{v}\|_{2,1}$ in Eq.~(\ref{3.5}) ensures that each cluster prototype is characterized by discriminative features. This implies that feature selection can guide data imputation by enhancing the feature representations of the cluster prototypes.

In addition, the second term ${\sum_{j\in \mathcal{N}^{v}(i)}S_{ji}^{v}\hat{\bm{X}}_{\cdot j}^{v}}$ in Eq.~(\ref{C2}) shows that CLIM-FS can further guide the imputation of sample $\bm{X}_{\cdot i}^{v}$ by taking a weighted average of its $k$ nearest neighbors within the corresponding view, where the weights are determined by pairwise sample similarities. Notably, although the consensus similarity matrix $\bm{H}$ across views does not appear explicitly in Eq.~(\ref{C2}), it plays an important role in facilitating the learning of $\bm{F}^*$ by capturing cross-view local geometric structures, thereby indirectly steering the data imputation.  This effect will be further demonstrated in the following analysis.

The above analysis illustrates how CLIM-FS leverages both the consensus clustering structure and the cross-view local geometric structure to guide data imputation. In what follows, we further demonstrate that the resulting imputations can, in turn, enhance feature selection performance. Previous studies have established that capturing both the cluster structure and the local geometric structure of data is crucial for effective unsupervised feature selection~\cite{X.W.LiuTNNLS2013,B.L.ChenTPAMI2022,LiZTKDE2013}. Accordingly, we present two theorems showing that the proposed method preserves these two types of data structures after imputation, thereby improving feature selection performance in the incomplete multi-view scenario. The statements and proofs of these two theorems are presented below.
\begin{theorem}\label{T1}
    For any two samples $\bm{X}_{\cdot i}^{v}$ and $\bm{X}_{\cdot j}^{v}$ with missing values, their imputed data $\hat{\bm{X}}_{\cdot i}^{v}$ and $\hat{\bm{X}}_{\cdot j}^{v}$, as obtained from Eq.~(\ref{C2}),  satisfy the following:
    \begin{enumerate}
    \item If the two samples belong to the same cluster, i.e., $\bm{F}_{i\cdot}^{*}=\bm{F}_{j\cdot}^{*}$, then 
    \begin{equation}
    \|\hat{\bm{X}}_{\cdot i}^{v} - \hat{\bm{X}}_{\cdot j}^{v}\|_{2} \leq \mu=\frac{1}{2}\sigma_{max}(\bm{W}^{v})\|\bm{F}^{v}\|_{1} + 1. 
    \end{equation}

    \item If the two samples belong to different clusters, i.e., $\bm{F}_{i\cdot}^{*} \neq \bm{F}_{j\cdot}^{*}$, and let $\Delta_{\bm{F}^{*}} = \|\bm{F}_{i \cdot}^{* \top} \!\!-\! \bm{F}_{j \cdot}^{* \top}\|_{2}$, then 
    \begin{equation}
    \|\hat{\bm{X}}_{\cdot i}^{v} \!-\! \hat{\bm{X}}_{\cdot j}^{v}\|_{2} \!\geq\! \nu\!=\!\frac{1}{2}\sigma_{min}(\bm{W}^{v})(\Delta_{\bm{F}^{*}}-\|\bm{F}^{v}\|_{1})-1,
     \end{equation}
     where $\nu > \mu $ when $\|\bm{F}^{v}\|_{1} < \frac{\sigma_{min}(\bm{W}^{v})\Delta_{\bm{F}^{*}} -4}{\sigma_{min}(\bm{W}^{v}) + \sigma_{max}(\bm{W}^{v})}$.
    \end{enumerate}
\end{theorem}

\begin{proof}
    According to Eq.~(\ref{C2}), the difference between $\hat{\bm{X}}_{\cdot i}^{v}$ and $\hat{\bm{X}}_{\cdot j}^{v}$ is given by 
    \begin{equation}\label{D1}
    \hat{\bm{X}}_{\cdot i}^{v} - \hat{\bm{X}}_{\cdot j}^{v} = \frac{1}{2}\bm{W}^{v}(\bm{F}_{i \cdot}^{v \top}\!-\bm{F}_{j \cdot}^{v \top} \!+  \bm{F}_{i \cdot}^{* \top} \!- \bm{F}_{j \cdot}^{* \top} ) + \tilde{\epsilon},
    \end{equation}
    where $\tilde{\epsilon}=\epsilon_{i}-\epsilon_{j}$, with $\epsilon_{i}$ defined as $\epsilon_{i}\!\!=\!\!\frac{1}{2}{\sum_{j\in \mathcal{N}^{v}(i)}S_{ji}^{v}\hat{\bm{X}}_{\cdot j}^{v}}$.  Without loss of generality, we assume  $\|\hat{\bm{X}}_{\cdot j}^{v}\|_{2}= 1, \forall j\in\{1,\dots,n\}$, which can be achieved by normalizing $\hat{\bm{X}}^{v}$.

    Next, we analyze two cases. In the first case, if two samples belong to the same cluster, i.e., $\bm{F}_{i\cdot}^{*}=\bm{F}_{j\cdot}^{*}$, we have
    \begin{equation}\label{D2}
    \begin{aligned}
    &\|\hat{\bm{X}}_{\cdot i}^{v} \!-\! \hat{\bm{X}}_{\cdot j}^{v}\|_{2}\leq \|\frac{1}{2}\bm{W}^{v}(\bm{F}_{i \cdot}^{v \top}\!\!-\!\bm{F}_{j \cdot}^{v \top} \!\!+\!  \bm{F}_{i \cdot}^{* \top} \!\!-\! \bm{F}_{j \cdot}^{* \top} )\|_{2} \!+\! \|\tilde{\epsilon}\|_{2}\\
    &\leq\frac{1}{2}\sigma_{max}(\bm{W}^{v})(\|\bm{F}_{i \cdot}^{v}\|_{2}+\|\bm{F}_{j \cdot}^{v}\|_{2}) + \|\epsilon_{i}\|_{2}+\|\epsilon_{j}\|_{2}\\
    &\leq\frac{1}{2}\sigma_{max}(\bm{W}^{v})\|\bm{F}^{v}\|_{1} + 1.
    \end{aligned}
    \end{equation}
    where $\sigma_{max}(\bm{W}^{v})$ denotes the largest singular values of $\bm{W}^{v}$.

    In the second case, if two samples belong to different clusters, i.e., $\bm{F}_{i\cdot}^{*}\neq\bm{F}_{j\cdot}^{*}$, we obtain
    
    \begin{equation}\label{D4}
    \begin{aligned}
    &\|\hat{\bm{X}}_{\cdot i}^{v} \!-\! \hat{\bm{X}}_{\cdot j}^{v}\|_{2}\geq \|\frac{1}{2}\bm{W}^{v}(\bm{F}_{i \cdot}^{v \top}\!\!-\!\bm{F}_{j \cdot}^{v \top} \!\!+\!  \bm{F}_{i \cdot}^{* \top} \!\!-\! \bm{F}_{j \cdot}^{* \top} )\|_{2} \!-\! \|\tilde{\epsilon}\|_{2}\\
    &\geq\frac{1}{2}\sigma_{min}(\bm{W}^{v})\|\bm{F}_{i \cdot}^{v \top}\!\!\!-\!\bm{F}_{j \cdot}^{v \top} \!\!+\!  \bm{F}_{i \cdot}^{* \top} \!\!\!-\! \bm{F}_{j \cdot}^{* \top}\|_{2}\!-\!(\|\epsilon_{i}\|_{2}\!+\!\|\epsilon_{j}\|_{2})\\
    &\geq\frac{1}{2}\sigma_{min}(\bm{W}^{v})(\|\bm{F}_{i \cdot}^{* \top} \!\!-\! \bm{F}_{j \cdot}^{* \top}\|_{2}-\|\bm{F}_{i \cdot}^{v \top}\!\!-\!\bm{F}_{j \cdot}^{v \top}\|_{2})-1\\
    &\geq\frac{1}{2}\sigma_{min}(\bm{W}^{v})(\Delta_{\bm{F}^{*}}-\|\bm{F}^{v}\|_{1})-1\\
    \end{aligned}
    \end{equation}
    where $\sigma_{min}(\bm{W}^{v})$ denotes the smallest singular values of $\bm{W}^{v}$, and $\Delta_{\bm{F}^{*}} = \|\bm{F}_{i \cdot}^{* \top} \!\!-\! \bm{F}_{j \cdot}^{* \top}\|_{2}$.
    
    Assuming that $\|\bm{F}^{v}\|_{1} < \frac{\sigma_{min}(\bm{W}^{v})\Delta_{\bm{F}^{*}} -4}{\sigma_{min}(\bm{W}^{v}) + \sigma_{max}(\bm{W}^{v})}$, it follows that 
    \begin{equation}\label{D6}
        \frac{1}{2}\sigma_{min}(\bm{W}^{v})(\Delta_{\bm{F}^{*}}-\|\bm{F}^{v}\|_{1})-1 \!>\! \frac{1}{2}\sigma_{max}(\bm{W}^{v})\|\bm{F}^{v}\|_{1} + 1.
    \end{equation}
    Since $\bm{F}^{v}$ is enforced to be sparse via the $\ell_{1}$-norm regularization  term $\beta\| \bm{F}^{v} \|_{1}$, we can adjust the regularization parameter $\beta$ to ensure that $\|\bm{F}^{v}\|_{1} < \frac{\sigma_{min}(\bm{W}^{v})\Delta_{\bm{F}^{*}} -4}{\sigma_{min}(\bm{W}^{v}) + \sigma_{max}(\bm{W}^{v})}$ is satisfied. 
\end{proof}

Theorem~\ref{T1} shows that the imputed data preserves the cluster structure, keeping samples within the same cluster close to each other after imputation, while those from different clusters remain well separated. As a result, the imputation process can improve the performance of subsequent feature selection.

\begin{theorem}\label{T2}
    The imputed data $\{\hat{\bm{X}}^{v}\}_{v=1}^{V}$ obtained from Eq.~(\ref{C2}) has the following properties:
    \begin{enumerate}
    \item For any  pair of samples $(\bm{X}_{\cdot i}^{v},\bm{X}_{\cdot j}^{v})$ with missing values from the same view, if their within-view similarity satisfies $S_{ij}^{v} \geq \varrho$, then there exists $\omega_{1}(\varrho)$ such that
    \begin{equation}
    \|\hat{\bm{X}}_{\cdot i}^{v} - \hat{\bm{X}}_{\cdot j}^{v}\|_{2} \leq \omega_{1}, \frac{d\omega_{1}}{d\varrho} <0.
    \end{equation}

    \item For any pair of samples $(\bm{X}_{\cdot i}^{v},\bm{X}_{\cdot j}^{m})$ with missing values from different views, if their consensus similarity satisfies $H_{ij} \geq \zeta$, then there exists $\omega_{2}(\zeta)$ and $\omega_{3}(\zeta)$ such that
    \begin{equation}
    \| \bm{F}^*_{i \cdot} - \bm{F}^*_{j \cdot} \|_2^2 \leq \omega_{2},\|\hat{\bm{X}}_{\cdot i}^{v} - \hat{\bm{X}}_{\cdot j}^{m}\|_{2} \leq \omega_{3}, \frac{d\omega_{3}}{d\zeta} <0
    \end{equation}
    \end{enumerate}
\end{theorem}
\begin{proof}
    1) According to Eq.~(\ref{C2}), we have
    \begin{equation}
    \Delta_{ij}^{v}= \|\frac{1}{2}S_{ji}^{v}\hat{\bm{X}}_{\cdot j}^{v}-\hat{\bm{X}}_{\cdot j}^{v}+\frac{1}{2}\sum_{p=1,p \neq j}^{n}S_{pi}^{v}\hat{\bm{X}}_{\cdot p}^{v}+\frac{1}{2}\bm{C}_{\cdot i}^{v}\|_{2},
    \end{equation}
    where $\Delta_{ij}^{v}=\|\hat{\bm{X}}_{\cdot i}^{v} - \hat{\bm{X}}_{\cdot j}^{v}\|_{2}$.

    Since $S_{ji}^{v} \geq \varrho$, we write $S_{ji}^{v}$ as $S_{ji}^{v}=\varrho+\eta$, where $\eta \geq 0$. Thus, $\sum_{p=1,p \neq j}^{n}S_{pi}^{v}=1-\varrho-\eta$. Then, the following inequality holds:
    \begin{equation}
    \begin{aligned}
    \Delta_{ij}^{v}&\!=\!\|(\frac{\varrho}{2}-1)\hat{\bm{X}}_{\cdot j}^{v}+\frac{1}{2}\eta\hat{\bm{X}}_{\cdot j}^{v}+\frac{1}{2}\sum_{p \neq j}S_{pi}^{v}\hat{\bm{X}}_{\cdot p}^{v}+\frac{1}{2}\bm{C}_{\cdot i}^{v}\|_{2}\\
    &\!\leq\!(1\!\!-\!\!\frac{\varrho}{2})\|\hat{\bm{X}}_{\cdot j}^{v}\|_{2}\!+\!\frac{1}{2}(\eta\|\hat{\bm{X}}_{\cdot j}^{v}\|_{2}\!+\!\sum_{p \neq j}S_{pi}^{v}\|\hat{\bm{X}}_{\cdot p}^{v}\|_{2}\!+\!\|\bm{C}_{\cdot i}^{v}\|_{2}).\\
    \end{aligned}
    \end{equation}

   Similar to the derivation of Theorem 1, we assume that $\|\hat{\bm{X}}_{\cdot j}^{v}\|_{2}=1$ for all $j \in {1, \dots, n}$. Therefore, by defining $\omega_{1}(\varrho)=\frac{3}{2}-\varrho+\frac{1}{2}\|\bm{C}_{\cdot i}^{v}\|_{2}$, we obtain
    \begin{equation}\label{H1}
    \Delta_{ij}^{v}= \|\hat{\bm{X}}_{\cdot i}^{v} - \hat{\bm{X}}_{\cdot j}^{v}\|_{2} \leq \omega_{1}.
    \end{equation}
    Moreover, we have ${d\omega_{1}}/{d\varrho}=-1<0$. 

    2) Since $\bm{F}^*$ is the optimal solution to problem~(\ref{4.11}), it minimizes the objective function over the feasible set. Let $\mathcal{J}$ denote the objective value of problem~(\ref{4.11}) at $\bm{F}^*$. Then, we have
    \begin{equation}
    \begin{aligned}
    &\operatorname{Tr}(\bm{F}^{*\top} \bm{L}_H \bm{F}^*) = \frac{1}{2} \sum_{p,q=1}^{n} H_{pq} \| \bm{F}^*_{p \cdot} - \bm{F}^*_{q \cdot} \|_2^2 \leq \mathcal{J}.\\
    \end{aligned}
    \end{equation}

    Given that $\mathnormal{H}_{pq}\geq 0, \forall p,q \in \{1,\dots,n\}$ and $H_{ij} \geq \zeta$, it follows that
    \begin{equation}\label{H2}
    \begin{aligned}
    H_{ij} \| \bm{F}^*_{i \cdot} - \bm{F}^*_{j \cdot} \|_2^2 \leq 2\mathcal{J}
    \Rightarrow \| \bm{F}^*_{i \cdot} - \bm{F}^*_{j \cdot} \|_2^2 \leq \omega_{2}(\zeta)=\frac{2\mathcal{J}}{\zeta}.
    \end{aligned}
    \end{equation}
    Eq.~(\ref{H2}) shows that $\| \bm{F}^*_{i \cdot} - \bm{F}^*_{j \cdot} \|_2^2$ is upper-bounded by a quantity that is inversely proportional to $H_{ij}$. Consequently, a larger similarity $H_{ij}$ encourages greater consistency between the clustering labels of samples $\bm{X}_{\cdot i}^{v}$ and $\bm{X}_{\cdot j}^{m}$.

    Analogous to the derivation in Theorem~\ref{T1}, we can obtain the following inequality:
    \begin{equation}\label{H3}
    \begin{aligned}
    \|\hat{\bm{X}}_{\cdot i}^{v} - \hat{\bm{X}}_{\cdot j}^{m}\|_{2} \leq \omega_{3}(\zeta)= \|\bm{W}^{v}\|_{2}(\| \bm{F}^{v}\|_{1}+\|\bm{F}^{m}\|_{1} + \frac{2\mathcal{J}}{\zeta}).
    \end{aligned}
    \end{equation}
    In addition, we have ${d\omega_{3}}/{d\zeta}=-\frac{2\|\bm{W}^{v}\|_{2}\mathcal{J}}{\zeta^2}<0$. 
\end{proof}

Theorem~\ref{T2} demonstrates that, for any two samples within the same view, higher within-view similarity results in a smaller distance between their imputed values. Furthermore, for any two samples from different views, higher cross-view similarity leads to greater consistency in clustering labels and a smaller distance between their imputed data. These properties indicate that the imputation results help preserve the local geometric structure, which is crucial for enhancing feature selection performance in the incomplete multi-view scenario.

\subsection{Analysis of Algorithm Convergence}\label{ConvergenceAnalysis}
Since problem~(\ref{3.5}) is not jointly convex with respect to all variables, we propose Algorithm 1 to optimize it by iteratively solving seven subproblems, namely Eqs. (\ref{4.5}), (\ref{4.8}), (\ref{4.11}), (\ref{4.17}), (\ref{4.19}), (\ref{4.24}), and (\ref{4.1}). To prove the convergence of Algorithm 1, we demonstrate that each subproblem converges monotonically. To this end, we first present Theorem~\ref{Convergence1}, which shows that updating $\bm{W}^{(v)}$ while keeping the other variables fixed ensures a monotonic decrease in the objective function of Eq.~(\ref{4.5}).

\begin{theorem}\label{Convergence1}
    The update rule for $\bm{W}^{(v)}$ in Algorithm 1 guarantees that the objective function in Eq.~(\ref{4.5}) is non-increasing.
\end{theorem}

\begin{proof}
    Let $\widehat{\bm{W}}^{v}$ represent the updated $\bm{W}^{v}$ at the current iteration. Since $\widehat{\bm{W}}^{v}$ is the optimal solution to problem~(8), we have
    \begin{equation}\label{5.1}
    \begin{aligned}
    &\|\hat{\bm{X}}^{v} - \widehat{\bm{W}}^{v}(\bm{F}^{v} + \bm{F}^{*})^{\top} \|_{\mathrm{F}}^{2} + \lambda\operatorname{Tr}(\widehat{\bm{W}}^{v\top}\bm{D}^{v}\widehat{\bm{W}}^{v})\\
    &\leq\|\hat{\bm{X}}^{v} - \bm{W}^{v}(\bm{F}^{v} + \bm{F}^{*})^{\top} \|_{\mathrm{F}}^{2} + \lambda\operatorname{Tr}(\bm{W}^{v\top}\bm{D}^{v}\bm{W}^{v}).
    \end{aligned}
    \end{equation}

    According to~\cite{F.P.NieNIPS2010}, for any two nonzero vectors $\bm{p}$ and $\bm{q}$, the inequality $\|\bm{p}\|_{2}-\frac{\|\bm{p}\|_{2}^{2}}{2\|\bm{q}\|_{2}} \leq \|\bm{q}\|_{2}-\frac{\|\bm{q}\|_{2}^{2}}{2\|\bm{q}\|_{2}}$ holds, and thus we obtain

    \begin{equation}\label{5.2}
    \begin{aligned}
    & \|\widehat{\bm{W}}_{i \cdot}^{v}\|_2- \frac{\|\widehat{\bm{W}}_{i \cdot}^{v}\|_{2}^{2}}{2\|{\bm{W}}_{i \cdot}^{v}\|_2} \leq \|\bm{W}_{i \cdot}^{v}\|_2- \frac{\|\bm{W}_{i \cdot}^{v}\|_{2}^{2}}{2\|\bm{W}_{i \cdot}^{v}\|_2}\\
    \Rightarrow &\sum_{i=1}^{d_{v}}(\|\widehat{\bm{W}}_{i \cdot}^{v}\|_2-\frac{\|\widehat{\bm{W}}_{i \cdot}^{v}\|_{2}^{2}}{2\|{\bm{W}}_{i \cdot}^{v}\|_2}) \leq \sum_{i=1}^{d_{v}}(\|\bm{W}_{i \cdot}^{v}\|_2-\frac{\|\bm{W}_{i \cdot}^{v}\|_{2}^{2}}{2\|\bm{W}_{i \cdot}^{v}\|_2}). \\
    \end{aligned}
    \end{equation}

    Since $\sum_{i=1}^{d_{v}}{\|{\bm{W}}_{i \cdot}^{v}\|_{2}^{2}}/{2\|{\bm{W}}_{i \cdot}^{v}\|_2}=\operatorname{Tr}(\bm{W}^{v\top}\!\bm{D}^{v}\bm{W}^{v}$), we can rewrite Eq.~(\ref{5.2}) as follows:
    \begin{equation}\label{5.3}
    \|\widehat{\bm{W}}^{v}\|_{2,1}\!-\!\operatorname{Tr}(\widehat{\bm{W}}^{v\top}\!\bm{D}^{v}\widehat{\bm{W}}^{v})\!\leq\!\|\bm{W}^{v}\|_{2,1}\!-\!\operatorname{Tr}(\bm{W}^{v\top}\!\bm{D}^{v}\bm{W}^{v}).
    \end{equation}

    By combining Eq.~(\ref{5.1}) and Eq.~(\ref{5.3}), we have
    \begin{equation}\label{5.4}
    \begin{aligned}
    &\|\hat{\bm{X}}^{v} - \widehat{\bm{W}}^{v}(\bm{F}^{v} + \bm{F}^{*})^{\top} \|_{\mathrm{F}}^{2} + \lambda\|\widehat{\bm{W}}^{v}\|_{2,1}\\
    &\leq\|\hat{\bm{X}}^{v} - \bm{W}^{v}(\bm{F}^{v} + \bm{F}^{*})^{\top} \|_{\mathrm{F}}^{2} + \lambda\|\bm{W}^{v}\|_{2,1}.
    \end{aligned}
    \end{equation}
    Hence, updating $\bm{W}^{(v)}$ leads to a monotonic decrease in the objective value in Eq.~(7).
\end{proof}

In addition, since the optimization problem w.r.t. $\bm{F}^{v}$ in Eq.~(\ref{4.8}) is convex and its gradient $\nabla(\bm{F}^{v}) = -2\hat{\bm{X}}^{v\top}\bm{W}^{v} + 2(\bm{F}^{v} + \bm{F}^{*})\bm{W}^{v\top}\bm{W}^{v}$ is Lipschitz continuous, the proximal gradient descent method is guaranteed to converge when updating $\bm{F}^{(v)}$~\cite{Beck A}. Furthermore, the updates for $\bm{F}^{*}$, $\bm{S}^{v}$, $\bm{H}$ and $\hat{\bm{X}}^{v}$ are guaranteed to converge, as each admits a closed-form solution provided in Eqs. (\ref{4.14}), (\ref{4.18.5}), (\ref{4.23}), (\ref{4.4}), respectively. In addition, updating $\alpha$ by solving the quadratic objective function in Eq. (\ref{4.24}) using the interior-point method has been proven to converge, as shown in~\cite{Boyd S P}. Therefore, with each iteration, the alternating optimization in Algorithm 1 monotonically decreases the objective function in Eq. (\ref{3.5}) until convergence is achieved. The convergence behavior of Algorithm~1 will be empirically validated in the experimental section.
\subsection{Time Complexity Analysis} 
For each iteration of Algorithm 1, the computational complexity of updating $\bm{W}^{v}$ is $\mathcal{O}(n{d}_{v}{c})$. The computational costs for updating $\bm{F}^{v}$ and $\bm{F}^{*}$ are $\mathcal{O}(n{d}_{v}c)$ and $\mathcal{O}(ndc+{n}^{2}{c})$, respectively, where $d=\sum_{v=1}^{V}{d}_{v}$. Updating $\hat{\bm{X}}^{v}$ requires $\mathcal{O}(n^{2}{d}_{v})$ operations, while updating $\bm{\alpha}$ has a complexity of $\mathcal{O}(ncV^{2})$. The updates for $\bm{S}^{v}$ and $\bm{H}$ involve only element-wise operations, so their computational cost is negligible.  In summary, the computational complexity of each iteration in Algorithm 1 is $\mathcal{O}(ndc+n^{2}d+ ncV^{2})$.

\begin{table}[t]
\centering
\caption{Statistics of different multi-view datasets}\label{Data}
\vspace*{-0.2cm}
\setlength{\tabcolsep}{1.4mm}
    \begin{tabular}{@{\extracolsep{\fill}}lcccc}
            \toprule
            Datasets  & Samples & Features & Views & Classes\\
            \midrule
            MSRC      & 210 &24/576/512/256/254 & 5 & 7 \\

            ORL    & 400  & 512/59/864/254 & 4 & 40 \\

            BBCSport  & 544 &3183/3203 & 2 & 5 \\

            Reuters      & 1200 &2000/2000/2000/2000/2000 & 5 & 6 \\
			
            100leaves      & 1600 & 64/64/64 & 3 & 100 \\

            LandUse    & 2100 & 20/59/40 & 3 & 21 \\

            NUS  & 2400  & 64/144/73/128/225/500 & 6 & 12 \\

            Aloi & 11025 & 73/13/64/64 & 4 & 100 \\
            \bottomrule
    \end{tabular}
\end{table}

\section{Experiments}\label{sec:Experimental}

\subsection{Experimental Schemes}
\subsubsection{Datasets} In this work, we evaluate the performance of the proposed method using eight real-world multi-view datasets: three object recognition datasets (MSRC\footnotemark[1], NUS~\cite{Z.X.HuIF2020}, and Aloi\footnotemark[2]), one face image dataset (ORL\footnotemark[1]), two text datasets (BBCSport\footnotemark[1] and Reuters\footnotemark[1]), one plant leaf image dataset (100leaves\footnotemark[1]), and one satellite image dataset (LandUse\footnotemark[3]). The statistical details of these datasets are summarized in Table~\ref{Data}. To simulate different missing data scenarios in our experiments, we adopt the following approaches. For the view-missing scenario, following ~\cite{H.Tao2019}, we randomly select $\delta\%$ of the samples and randomly remove one view from each selected sample. For the variable-missing scenario, according to~\cite{J.X.You2020}, we randomly remove $\delta\%$ of the entries from the data matrix of each view. For the mixed-missing scenario, we first randomly select $\delta\%$ of the samples and randomly remove one view from each selected sample; then, for each view, we randomly remove $\delta\%$ of the entries from the data matrix of the remaining samples. In all experiments, we vary $\delta$ within ${10\%, 20\%, 30\%, 40\%, 50\%}$ to simulate different degrees of incompleteness for each multi-view dataset.

\footnotetext[1]{https://gitee.com/zhangfk/multi-view-dataset}

\footnotetext[2]{\url{https://elki-project.github.io/datasets/multi_view}}

\footnotetext[3]{http://weegee.vision.ucmerced.edu/datasets/landuse.html}

\begin{table*}[t]
\centering
\caption{In the mixed-missing scenario, means (\%) of ACC and NMI of different methods on eight multi-view datasets.}
\label{Mixed-missing results}
\resizebox{\textwidth}{!}{
\begin{tabular}{lcccccccccccccccc}
\toprule
\multirow{2}{*}{Methods} 
& \multicolumn{2}{c}{MSRC} & \multicolumn{2}{c}{ORL} 
& \multicolumn{2}{c}{BBCSport} & \multicolumn{2}{c}{Reuters} 
& \multicolumn{2}{c}{100leaves} & \multicolumn{2}{c}{LandUse} 
& \multicolumn{2}{c}{NUS} & \multicolumn{2}{c}{Aloi} \\

\cmidrule(r){2-3} \cmidrule(r){4-5} \cmidrule(r){6-7} \cmidrule(r){8-9} 
\cmidrule(r){10-11} \cmidrule(r){12-13} \cmidrule(r){14-15} \cmidrule(r){16-17}
~ & ACC & NMI & ACC & NMI & ACC & NMI & ACC & NMI 
  & ACC & NMI & ACC & NMI & ACC & NMI & ACC & NMI \\
\midrule
CLIM-FS     & \textbf{70.34} & \textbf{62.02} & \textbf{52.94} & \textbf{69.27} & \textbf{41.54} & \textbf{7.78}  & \textbf{35.08} & \textbf{15.53} & \textbf{44.52} & \textbf{67.60} & \textbf{20.98} & \textbf{22.30} & \textbf{23.98} & \textbf{12.21} & \textbf{39.06} & \textbf{61.08} \\
AllFea       & 31.62 & 20.01 & 28.73 & 42.43 & 35.61 &  1.21 & 19.53 &  2.63 & 22.83 & 48.87 & 13.03 & 12.25 & 15.92 &  6.57 &  8.28 & 17.77 \\
CvLP\_DCL    & 55.35 & 46.02 & 29.74 & 43.52 & 36.35 &  2.25 & 26.21 &  7.52 & 21.05 & 47.93 & 14.61 & 12.61 & 22.58 &  9.31 & 19.12 & 41.80 \\
WLTL         & 57.76 & 51.48 & 30.87 & 45.75 & 36.31 &  1.98 & 26.57 &  7.43 & 20.87 & 48.01 & 15.26 & 15.63 & 21.57 &  9.61 & 15.27 & 35.31 \\
PTFS         & 56.38 & 48.79 & 33.77 & 48.28 & 36.11 &  1.50 & 26.66 &  7.19 & 21.98 & 49.28 & 15.69 & 15.89 & 21.79 &  9.87 & 16.52 & 39.76 \\
UKMFS        & 55.06 & 49.47 & 30.64 & 45.31 & 36.20 &  2.02 & 28.05 &  7.96 & 21.47 & 49.06 & 15.26 & 15.32 & 21.36 &  9.52 & 16.31 & 37.18 \\
CVFS         & 41.55 & 27.98 & 34.15 & 48.64 & 36.22 &  2.21 & 25.35 &  7.81 & 18.66 & 45.28 & 12.54 & 10.93 & 15.60 &  6.78 &  8.32 & 17.78 \\
C$^2$IMUFS   & 48.58 & 39.94 & 30.22 & 44.50 & 36.49 &  2.13 & 25.08 &  7.47 & 21.34 & 48.40 & 11.65 &  9.64 & 21.05 &  9.23 & 20.88 & 43.73 \\
UIMUFSLR     & 61.78 & 53.46 & 33.97 & 48.70 & 35.91 &  1.95 & 27.96 &  9.10 & 20.92 & 48.07 & 15.23 & 15.66 & 21.01 &  9.10 & 16.80 & 38.19 \\
UNIFIER      & 61.64 & 59.73 & 31.59 & 46.93 & 36.07 &  2.54 & 27.79 &  8.86 & 23.86 & 50.89 & 15.32 & 15.73 & 16.60 &  5.77 & 15.11 & 37.08 \\
TIME-FS      & 49.47 & 39.26 & 30.12 & 44.50 & 36.52 &  2.37 & 29.61 &  8.19 & 21.01 & 48.00 & 15.29 & 15.68 & 21.09 &  9.28 & 17.64 & 40.15 \\
TERUIMUFS    & 62.38 & 54.76 & 32.42 & 46.74 & 36.17 &  1.90 & 29.71 & 10.73 & 21.77 & 48.31 & 15.15 & 15.65 & 21.94 &  9.82 & 20.35 & 43.67 \\
\bottomrule
\end{tabular}
}
\end{table*}

\subsubsection{Comparison Methods}  To validate the effectiveness of the proposed CLIM-FS, we compare it with several state-of-the-art methods, as described below:
\begin{itemize}
    \item \textbf{AllFea}: It utilizes all original features.
    \item \textbf{CvLP\_DCL}~\cite{CvLP_DCL}: It learns cross-view similarity graphs to preserve the local structure of the data.
    \item \textbf{WLTL}~\cite{WLTL}: It integrates multi-view spectral clustering with weighted low-rank tensor learning to generate pseudo labels for feature selection.
    \item \textbf{PTFS}~\cite{PTFS}: It learns a unified tensor graph with block-diagonal constraints to simultaneously capture cross-view high-order correlations and discriminative partition information.
    \item \textbf{UKMFS}~\cite{R.Y.HuAAAI2025}: It employs binary hashing to obtain weakly-supervised labels, which guide feature selection and similarity graph learning.
    \item \textbf{CVFS}~\cite{CVFS}: It incorporates MUFS into a weighted non-negative matrix factorization model to select features from incomplete multi-view data.
    \item \textbf{C$^2$IMUFS}~\cite{C2IMUFS}: It exploits the complementary and consensus information across different views from incomplete data to learn complete similarity graphs for local structure preservation and feature selection.
    \item \textbf{UIMUFSLR}~\cite{UIMUFSLR}:  It exploits adaptive sample-weighted graph fusion to mitigate the impact of missing data, while enforcing a low-redundancy constraint in a low-dimensional space to obtain discriminative features.
    \item \textbf{UNIFIER}~\cite{UNIFIER}: It integrates multi-view feature selection and missing-view imputation into a joint learning framework.
    \item \textbf{TIME-FS}~\cite{TIME-FS}: It uses tensor CP decomposition to construct a consistent anchor graph and view preference weight matrix for incomplete multi-view feature selection.
    \item \textbf{TERUIMUFS}~\cite{TERUIMUFS}:  It jointly optimizes tensor low-rank representation, sample diversity regularization, and self-representation learning for unsupervised feature selection on unbalanced incomplete multi-view data.
\end{itemize}

To ensure a fair comparison, we use grid search to tune the hyperparameters of all competing methods according to the ranges recommended in their original papers, and report the optimal results.  For our method, the parameters  $\beta$ and $\lambda$ are selected via grid search over $\{10^{-3}, 10^{-2}, 10^{-1}, 1, 10, 10^2, 10^3\}$. Since determining the optimal number of selected features remains a challenging problem~\cite{J.D.Li2017}, we choose the proportion of selected features from $\{10\%, 20\%, 30\%, 40\%, 50\%\}$. Given that MUFS methods designed for complete data cannot be directly applied to incomplete multi-view data, we first impute the missing values using the mean of the available data before applying these methods. In addition, for IMUFS methods that can only address the view-missing issue, we first impute missing variables with the mean of the available data before applying these methods to mixed-missing and variable-missing scenarios. For performance evaluation, we follow the commonly used strategy in unsupervised feature selection~\cite{S.Solorio2020,X.DongTPAMI2025}, employing two clustering metrics, Clustering Accuracy (ACC) and Normalized Mutual Information (NMI), to assess the quality of features selected by different methods. We run the $k$-means clustering algorithm 50 times on the selected features and report the average results.

\begin{figure*}[!htbp]
\centering
\includegraphics[width=0.98\textwidth]{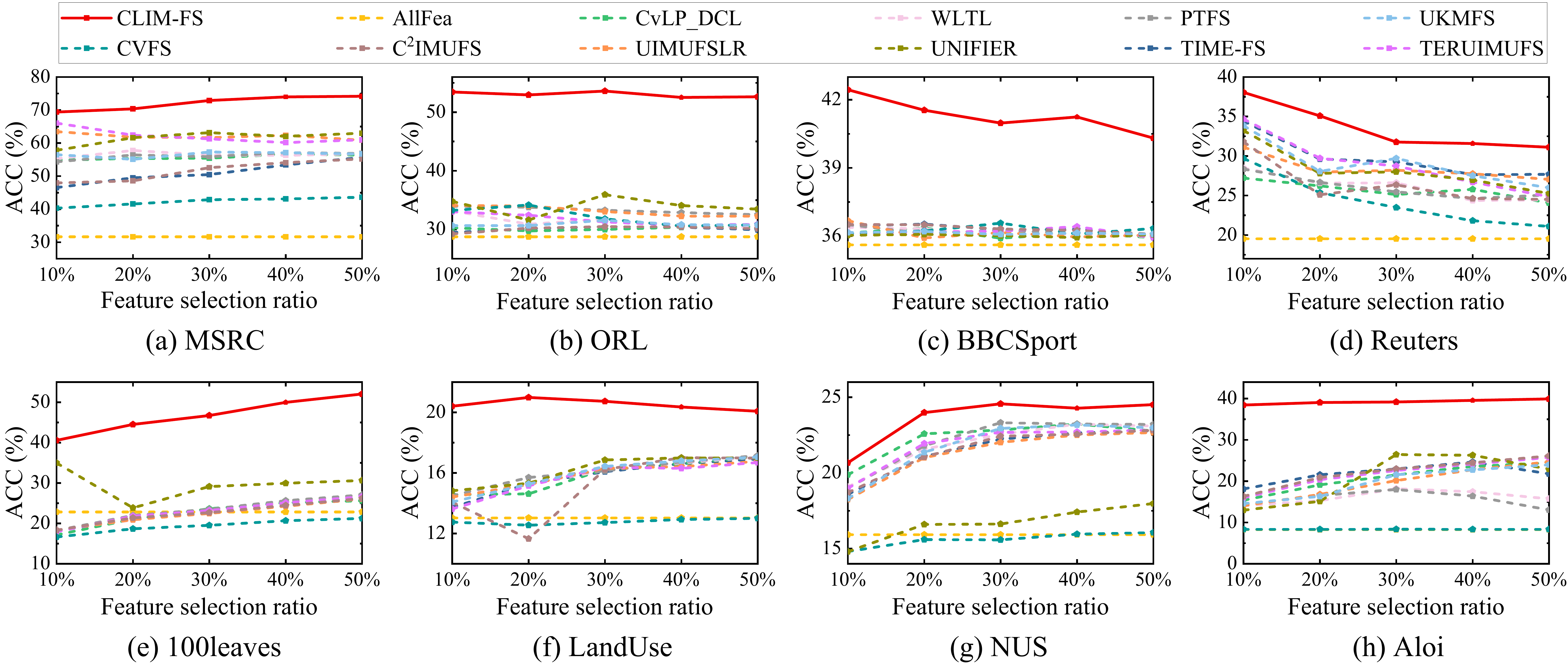}
\caption{ACC of different methods on eight multi-view datasets with different feature selection ratios in the mixed-missing scenario.}\label{ACC-feature-var}
\end{figure*}

\begin{figure*}[!htbp]
\centering
\includegraphics[width=0.98\textwidth]{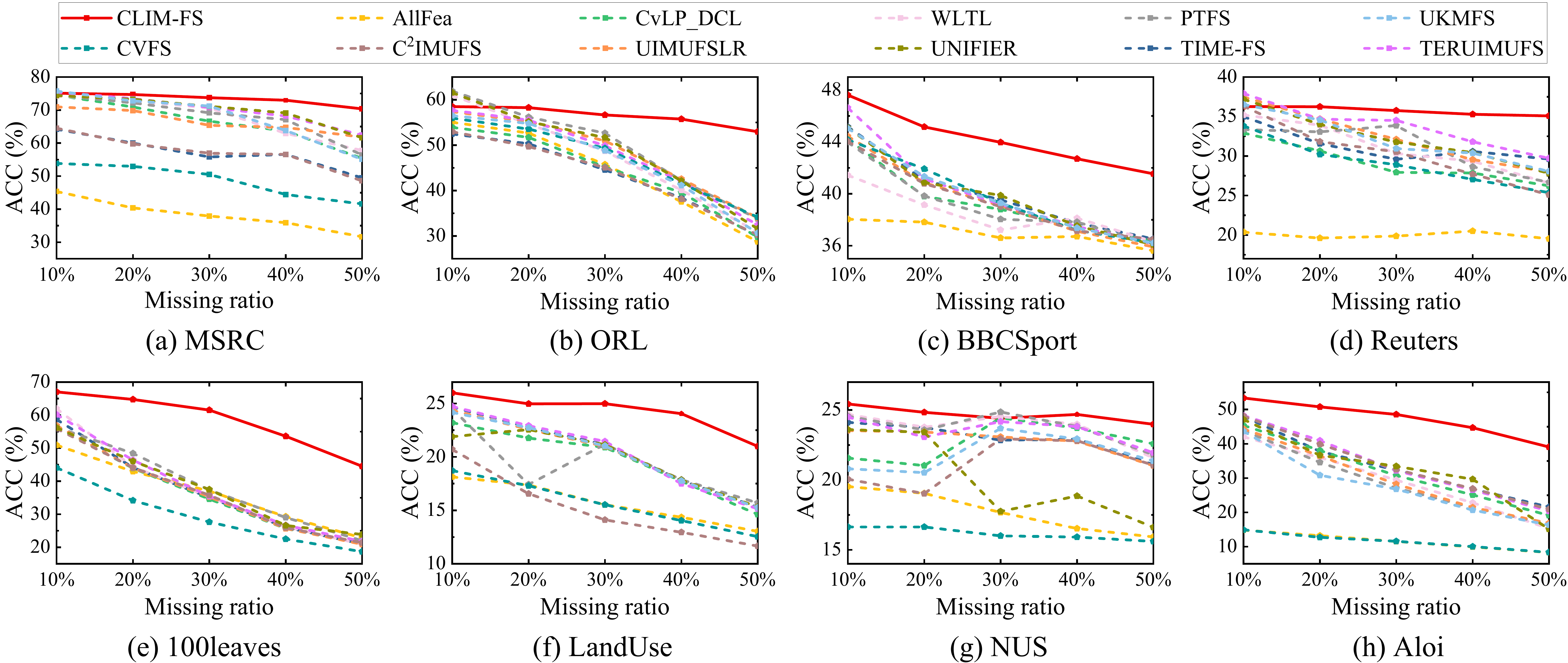}
\caption{ACC of different methods on eight multi-view datasets with different missing ratios in the mixed-missing scenario.}\label{ACC-missing-var}
\end{figure*}

\subsection{Performance Analyses}

\begin{table*}[t]
\centering
\caption{In the variable-missing scenario, means (\%) of ACC and NMI of different methods on eight multi-view datasets.}\label{Variable-missing results}
\resizebox{\textwidth}{!}{
\begin{tabular}{lcccccccccccccccc}
\toprule
\multirow{2}{*}{Methods} 
& \multicolumn{2}{c}{MSRC} & \multicolumn{2}{c}{ORL} 
& \multicolumn{2}{c}{BBCSport} & \multicolumn{2}{c}{Reuters} 
& \multicolumn{2}{c}{100leaves} & \multicolumn{2}{c}{LandUse} 
& \multicolumn{2}{c}{NUS} & \multicolumn{2}{c}{Aloi} \\

\cmidrule(r){2-3} \cmidrule(r){4-5} \cmidrule(r){6-7} \cmidrule(r){8-9} 
\cmidrule(r){10-11} \cmidrule(r){12-13} \cmidrule(r){14-15} \cmidrule(r){16-17}
~ & ACC & NMI & ACC & NMI & ACC & NMI & ACC & NMI 
  & ACC & NMI & ACC & NMI & ACC & NMI & ACC & NMI \\
\midrule

{CLIM-FS} & \textbf{75.81} & \textbf{68.24} & \textbf{61.20} & \textbf{79.54} & \textbf{47.03} & \textbf{15.95} & \textbf{33.59} & \textbf{14.67} & \textbf{69.36} & \textbf{86.09} & \textbf{26.26} & \textbf{32.52} & \textbf{26.06} & \textbf{13.01} & \textbf{27.10} & \textbf{53.47}\\

{AllFea} & 36.04 & 24.49 & 47.38 & 66.15 & 36.58 & 2.90 & 18.82 & 1.99 & 30.14 & 58.73 & 14.94 & 15.07 & 17.59 & 8.24 & 9.29 & 20.02 \\

{CvLP\_DCL} & 65.92 & 56.92 & 44.34 & 63.11 & 25.22 & 1.00 & 19.95 & 0.71 & 25.40 & 53.26 & 17.71 & 17.56 & 24.45 & 11.35 & 23.05 & 46.56 \\

{WLTL} & 68.04 & 60.40 & 45.34 & 64.46 & 25.60 & 1.11 & 20.43 & 2.84 & 27.54 & 56.45 & 17.62 & 19.88 & 23.37 & 10.57 & 25.16 & 50.72 \\

{PTFS} & 67.93 & 58.64 & 48.04 & 66.65 & 25.15 & 1.06 & 23.23 & 4.65 & 27.11 & 55.65 & 16.28 & 15.30 & 21.46 & 9.71 & 24.03 & 46.70 \\

{UKMFS} & 36.70 & 25.32 & 44.25 & 62.44 & 25.10 & 1.10 & 25.37 & 8.19 & 18.44 & 47.41 & 13.85 & 13.27 & 15.54 & 7.15 & 18.13 & 41.97 \\

{CVFS} & 45.41 & 35.99 & 47.64 & 66.72 & 39.03 & 6.02 & 26.01 & 8.34 & 23.16 & 50.97 & 13.89 & 13.44 & 15.65 & 7.16 & 18.61 & 43.94 \\

{C$^2$IMUFS} & 55.33 & 46.78 & 48.12 & 66.65 & 35.24 & 5.08 & 26.22 & 8.74 & 22.45 & 50.01 & 13.92 & 13.50 & 17.24 & 6.55 & 21.42 & 44.81 \\

{UIMUFSLR} & 66.69 & 60.54 & 47.10 & 66.54 & 35.84 & 1.10 & 20.31 & 3.07 & 27.27 & 55.50 & 17.36 & 19.98 & 22.66 & 10.62 & 19.73 & 44.11 \\

{UNIFIER} & 70.69 & 60.56 & 52.54 & 69.27 & 39.00 & 6.37 & 27.23 & 9.74 & 27.24 & 55.24 & 17.33 & 20.39 & 22.28 & 10.34 & 10.41 & 25.86 \\

{TIME-FS} & 65.77 & 57.78 & 46.02 & 65.56 & 27.01 & 1.13 & 22.26 & 3.77 & 26.92 & 55.01 & 17.00 & 19.50 & 22.13 & 9.09 & 22.03 & 45.98 \\

{TERUIMUFS} & 67.93 & 58.64 & 48.04 & 66.65 & 25.47 & 1.05 & 25.55 & 6.80 & 26.08 & 54.06 & 18.41 & 20.07 & 23.31 & 11.46 & 25.02 & 48.33 \\
\bottomrule
\end{tabular}}
\end{table*}

\begin{table*}[t]
\centering
\caption{In the view-missing scenario, means (\%) of ACC and NMI of different methods on eight multi-view datasets.}\label{View-missing results}
\resizebox{\textwidth}{!}{
\begin{tabular}{lccccccccccccccccc}
\toprule
\multirow{2}{*}{Methods} 
& \multicolumn{2}{c}{MSRC} & \multicolumn{2}{c}{ORL} 
& \multicolumn{2}{c}{BBCSport} & \multicolumn{2}{c}{Reuters} 
& \multicolumn{2}{c}{100leaves} & \multicolumn{2}{c}{LandUse} 
& \multicolumn{2}{c}{NUS} & \multicolumn{2}{c}{Aloi} \\

\cmidrule(r){2-3} \cmidrule(r){4-5} \cmidrule(r){6-7} \cmidrule(r){8-9} 
\cmidrule(r){10-11} \cmidrule(r){12-13} \cmidrule(r){14-15} \cmidrule(r){16-17}
~ & ACC & NMI & ACC & NMI & ACC & NMI & ACC & NMI 
  & ACC & NMI & ACC & NMI & ACC & NMI & ACC & NMI \\
\midrule

{CLIM-FS} & \textbf{63.13} & \textbf{53.57} & \textbf{58.01} & \textbf{75.70} & \textbf{41.12} & \textbf{7.01} & \textbf{37.73} & \textbf{17.67} & \textbf{54.89} & \textbf{75.49} & \textbf{21.92} & \textbf{24.38} & \textbf{24.80} & \textbf{12.56} & \textbf{44.94} & \textbf{66.08} \\

{AllFea} & 43.56 & 34.38 & 50.50 & 67.38 & 37.45 & 4.04 & 21.73 & 4.67 & 43.65 & 68.05 & 16.31 & 17.48 & 19.30 & 9.54 & 14.87 & 33.05 \\

{CvLP\_DCL} & 61.27 & 50.04 & 49.28 & 66.94 & 39.63 & 6.68 & 28.79 & 10.08 & 50.01 & 68.71 & 17.17 & 18.16 & 19.46 & 9.92 & 41.25 & 63.48 \\

{WLTL} & 60.39 & 50.75 & 53.86 & 71.22 & 38.39 & 4.79 & 31.67 & 13.69 & 50.93 & 69.25 & 18.04 & 19.93 & 20.56 & 9.17 & 40.48 & 60.97 \\

{PTFS} & 63.06 & 53.84 & 55.75 & 71.08 & 38.58 & 5.57 & 31.11 & 13.07 & 49.67 & 72.44 & 19.02 & 22.77 & 19.22 & 9.08 & 40.99 & 63.71 \\

{UKMFS} & 58.85 & 50.03 & 49.51 & 67.79 & 39.50 & 6.04 & 33.21 & 14.34 & 46.60 & 69.65 & 18.92 & 22.25 & 17.04 & 9.32 & 40.36 & 60.01 \\

{CVFS} & 51.52 & 41.51 & 51.82 & 68.97 & 40.48 & 6.79 & 32.36 & 13.86 & 39.28 & 64.78 & 18.61 & 20.84 & 16.67 & 8.34 & 14.62 & 33.10 \\

{C$^2$IMUFS} & 50.89 & 42.65 & 49.15 & 67.13 & 37.65 & 5.47 & 33.65 & 14.80 & 46.39 & 67.64 & 19.83 & 20.09 & 20.04 & 10.85 & 42.70 & 64.93 \\

{UIMUFSLR} & 60.43 & 53.92 & 55.41 & 71.68 & 39.36 & 6.49 & 33.97 & 13.73 & 45.91 & 69.28 & 20.03 & 22.09 & 24.02 & 11.81 & 40.68 & 63.15 \\

{UNIFIER} & 61.02 & 50.81 & 56.74 & 73.20 & 39.44 & 6.21 & 35.40 & 15.70 & 50.13 & 69.16 & 20.27 & 21.79 & 20.86 & 10.95 & 21.86 & 42.71 \\

{TIME-FS} & 60.86 & 51.61 & 52.76 & 69.89 & 38.90 & 6.20 & 32.12 & 11.73 & 49.83 & 72.13 & 17.48 & 21.45 & 19.59 & 9.13 & 42.16 & 64.38 \\

{TERUIMUFS} & 61.66 & 52.16 & 53.36 & 70.56 & 39.83 & 6.89 & 33.56 & 13.42 & 47.20 & 70.20 & 20.18 & 23.53 & 24.17 & 11.36 & 42.89 & 64.66 \\
\bottomrule
\end{tabular}
}
\end{table*}

In this section, we evaluate the performance of CLIM-FS by comparing it with other competing methods under various incomplete multi-view scenarios, including the mixed-missing scenario and its two specific cases: variable-missing and view-missing scenarios.
\subsubsection{Performance Comparison under Mixed-Missing Scenario}
Table~\ref{Mixed-missing results} summarizes the ACC and NMI values of all methods on eight multi-view datasets under the mixed missing scenario with a 50\% missing ratio and 20\% feature selection ratio. The best results are highlighted in bold. As shown in Table~\ref{Mixed-missing results}, CLIM-FS achieves the best performance on all datasets in terms of ACC and NMI. Specifically, on the MSRC, ORL, 100leaves, and Aloi datasets, CLIM-FS achieves an average improvement of over 17\% in both ACC and NMI. For the Reuters and LandUse datasets, CLIM-FS outperforms all competing methods by more than 6\% and 7\% in  ACC and NMI, respectively. 
As to the BBCSport and NUS datasets, CLIM-FS also achieves an average increase of more than 3\% in both ACC and NMI. In addition, CLIM-FS significantly outperforms the baseline AllFea on all eight datasets, demonstrating its effectiveness in reducing feature dimensionality.

Furthermore, to comprehensively evaluate the effectiveness of CLIM-FS, we present the performance of all methods under various feature selection ratios and missing ratios. Due to space constraints, only the experimental results for ACC are reported, while the results for NMI are presented in Figs. 1 and 2 of the supplementary materials. Fig. \ref{ACC-feature-var} shows the ACC values of all methods for different feature selection ratios, with the missing ratio fixed at 50\%. As illustrated in Fig. \ref{ACC-feature-var}, CLIM-FS outperforms the other methods in most cases when the feature selection ratio ranges from 10\% to 50\%. Additionally, Fig. \ref{ACC-missing-var} presents the ACC values of all methods across varying missing ratios, with the feature selection ratio fixed at 20\%, while the NMI results are provided in Fig. 2 of the supplementary materials. It can be observed that CLIM-FS consistently achieves the best performance in most cases.

\subsubsection{Performance Comparison under Variable-Missing and View-Missing Scenarios}
Tables~\ref{Variable-missing results} and~\ref{View-missing results} present the performance of all methods across eight datasets under the variable-missing and view-missing scenarios, respectively, with a 50\% missing ratio and 20\% feature selection ratio. In the variable-missing scenario, as shown in Table~\ref{Variable-missing results}, CLIM-FS outperforms other competitors across all datasets. Specifically, on the MSRC and 100leaves datasets, CLIM-FS achieves average improvements of over 17\% and 18\% in ACC and NMI, respectively. On the ORL, BBCSport, and LandUse datasets, CLIM-FS achieves more than a 10\% average improvement in both ACC and NMI. For the Reuters, NUS, and Aloi datasets, CLIM-FS surpasses other competitors, with average improvements of over 5\% in ACC and 3\% in NMI. 

Furthermore, as shown in Table~\ref{View-missing results}, CLIM-FS continues to achieve the best performance under the view-missing scenario. Specifically, on the Aloi dataset, CLIM-FS obtains average improvements of more than 10\% in both ACC and NMI. On the MSRC, ORL, Reuters, and 100leaves datasets, CLIM-FS achieves average improvements of over 5\% for both ACC and NMI. On the BBCSport, LandUse, and NUS datasets, CLIM-FS still yields the best results in terms of both ACC and NMI.

Based on the above analysis, it can be concluded that CLIM-FS consistently outperforms other methods under mixed-missing, variable-missing, and view-missing settings. The superior performance of CLIM-FS is attributed to the integration of feature selection and the imputation of missing views and variables into a unified learning framework. Furthermore, simultaneously leveraging both consistency and diversity across different views further enhances the effectiveness of CLIM-FS.
\begin{figure}[t]
\centering
\includegraphics[width=0.49\textwidth]{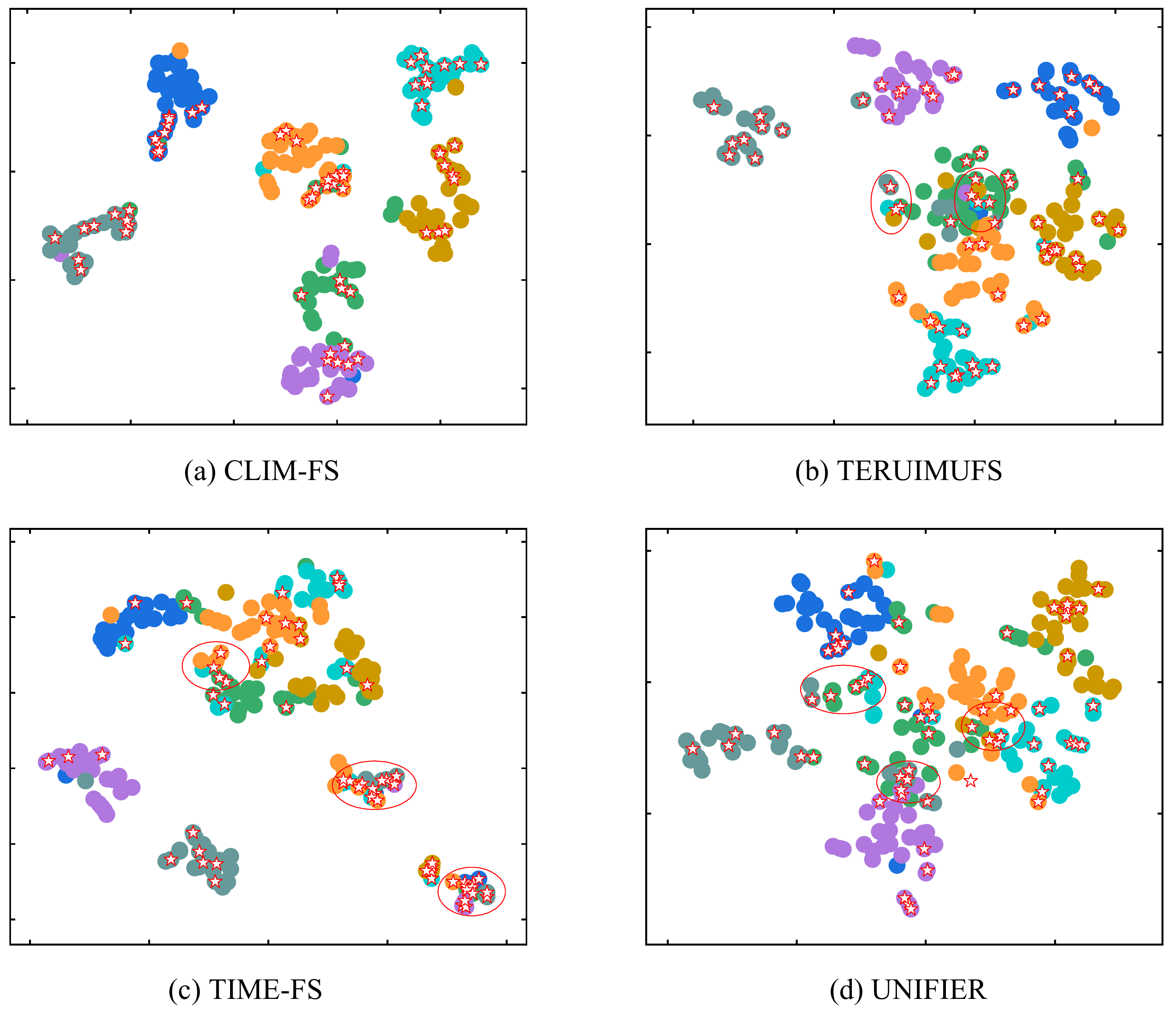}
\caption{t-SNE visualizations of features selected by four ``one-stage'' methods on MSRC dataset.}\label{tsne}
\end{figure}
\subsection{Visualization}\label{Visualization}
In Section~\ref{sec:Collaborative Learning Mechanism}, we theoretically demonstrate that the proposed CLIM-FS ensures that the imputed data preserve the original intra-cluster compactness and inter-cluster separability (Theorem~\ref{T1}), and that cross-view similarity-guided imputation preserves the local geometric structures induced by similarity (Theorem~\ref{T2}). In this section, we empirically validate these theoretical guarantees via a series of visualization experiments.

To this end, t-SNE is employed to project the selected features of the imputed MSRC dataset into a two-dimensional space. Fig.~\ref{tsne} shows the visualization results for the four ``one-stage'' methods at a 30\% missing ratio, with 20\% of the features selected. Missing instances are highlighted with red stars. As shown in Fig.~\ref{tsne}, the competing methods fail to effectively separate imputed samples belonging to different clusters, resulting in substantial overlap among these samples, as highlighted by the red circles in the figure. In contrast, CLIM-FS successfully distinguishes imputed samples from different clusters while maintaining compactness within each cluster. These results demonstrate that our joint imputation and feature selection method keeps samples within the same cluster tightly grouped after imputation, while clearly separating those from different clusters. Furthermore, we visualize the pairwise relationships among samples in the MSRC dataset using heatmaps, where relationships are measured by a Gaussian kernel applied to both the data imputed by four ``one-stage'' methods and the ground-truth MSRC data. As shown in Fig.~\ref{HeatMap}, the heatmap generated by CLIM-FS most closely resembles that of the ground-truth data, indicating that our method effectively preserves the proximity of similar samples after imputation. These results demonstrate that CLIM-FS can effectively preserve the cluster structure and local geometric structure in the incomplete multi-view scenario, which is beneficial for improving feature selection performance.

\begin{figure*}[t]
\centering
\includegraphics[width=0.98\textwidth]{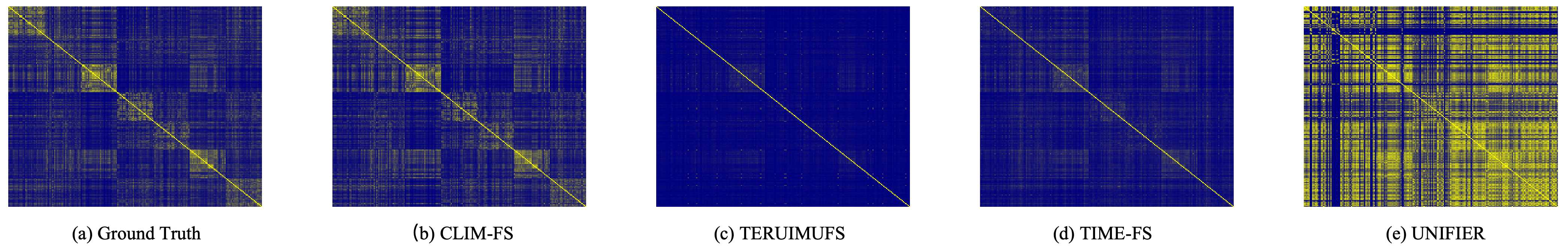}
\caption{Sample similarity structure visualizations on MSRC dataset: ground-truth vs. four ``one-stage'' methods.}\label{HeatMap}
\end{figure*}
\subsection{Convergence Analysis}\label{Convergence}
In Section~\ref{ConvergenceAnalysis}, we theoretically prove the convergence of Algorithm 1. In this section, we proceed to experimentally investigate the convergence behavior of Algorithm 1. Fig. \ref{convergence} illustrates the convergence curves of CLIM-FS on MSRC, Reuters, 100leaves and Aloi datasets, with the $x$-axis representing the iteration number and the $y$-axis denoting the objective values of CLIM-FS. As shown in Fig. \ref{convergence}, the objective value decreases monotonically and converges within approximately 40 iterations, demonstrating the effectiveness of the proposed algorithm.

\begin{figure}[t]
\centering
\includegraphics[width=0.48\textwidth]{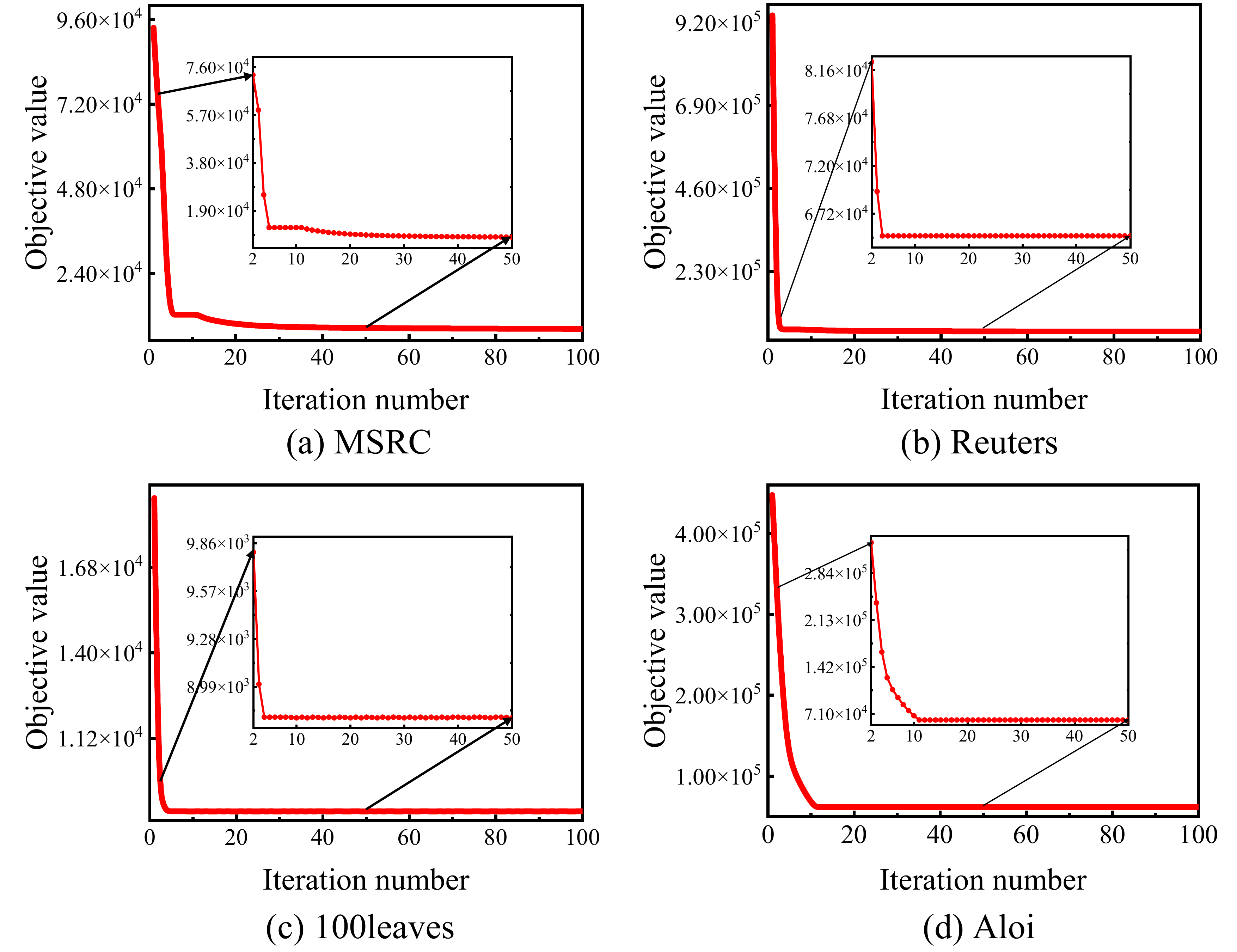}
\caption{Object function values of CLIM-FS at different iterations on MSRC, Reuters, 100leaves and Aloi datasets.}\label{convergence}
\end{figure}

\subsection{Parameter Sensitivity Analysis}
The objective function of CLIM-FS involves four parameters: $\lambda$, $\xi_v$, $\gamma$, and $\beta$. However, the regularization parameters $\xi_v$ and $\gamma$ are automatically determined during the optimization of $\bm{S}^{v}$ and $\bm{H}$, respectively. Therefore, we focus our investigation on the influence of $\lambda$ and $\beta$. Fig.~\ref{sensitivity} shows how ACC and NMI vary when one parameter is fixed and the other is varied.  The experimental results show that the performance of CLIM-FS is relatively sensitive to $\beta$, while it remains stable with respect to changes in $\lambda$. Moreover, CLIM-FS usually achieves prominent performances when $\beta$ is less than 10 and $\lambda$ exceeds 0.01. Hence,$\lambda$ and $\beta$ can be empirically fine-tuned within these ranges to obtain optimal results.

\subsection{Ablation Study}
In this section, we conduct an ablation study to evaluate the effectiveness of each component in CLIM-FS. Specifically, we compare the performance of CLIM-FS with its three variants: (\romannumeral1) CLIM-FS-\uppercase\expandafter{\romannumeral1}: the adaptive data imputation module is removed from Eq.~(\ref{3.5}), and missing values are instead filled with mean values. (\romannumeral2) CLIM-FS-\uppercase\expandafter{\romannumeral2}: the consensus cluster structure learning module is removed from Eq.~(\ref{3.5}). (\romannumeral3) CLIM-FS-\uppercase\expandafter{\romannumeral3}: the consistency and diversity-based similarity graph learning module is removed from Eq.~(\ref{3.5}). Figure~\ref{ablation} shows the ablation results in terms of ACC and NMI across eight multi-view datasets. The results demonstrate that CLIM-FS consistently outperforms all three variants on every dataset, highlighting the effectiveness of adaptive data imputation, cross-view cluster structure learning, and consistency- and diversity-based similarity graph learning in enhancing feature selection performance. 

\begin{figure}[t]
\centering
\includegraphics[width=0.48\textwidth]{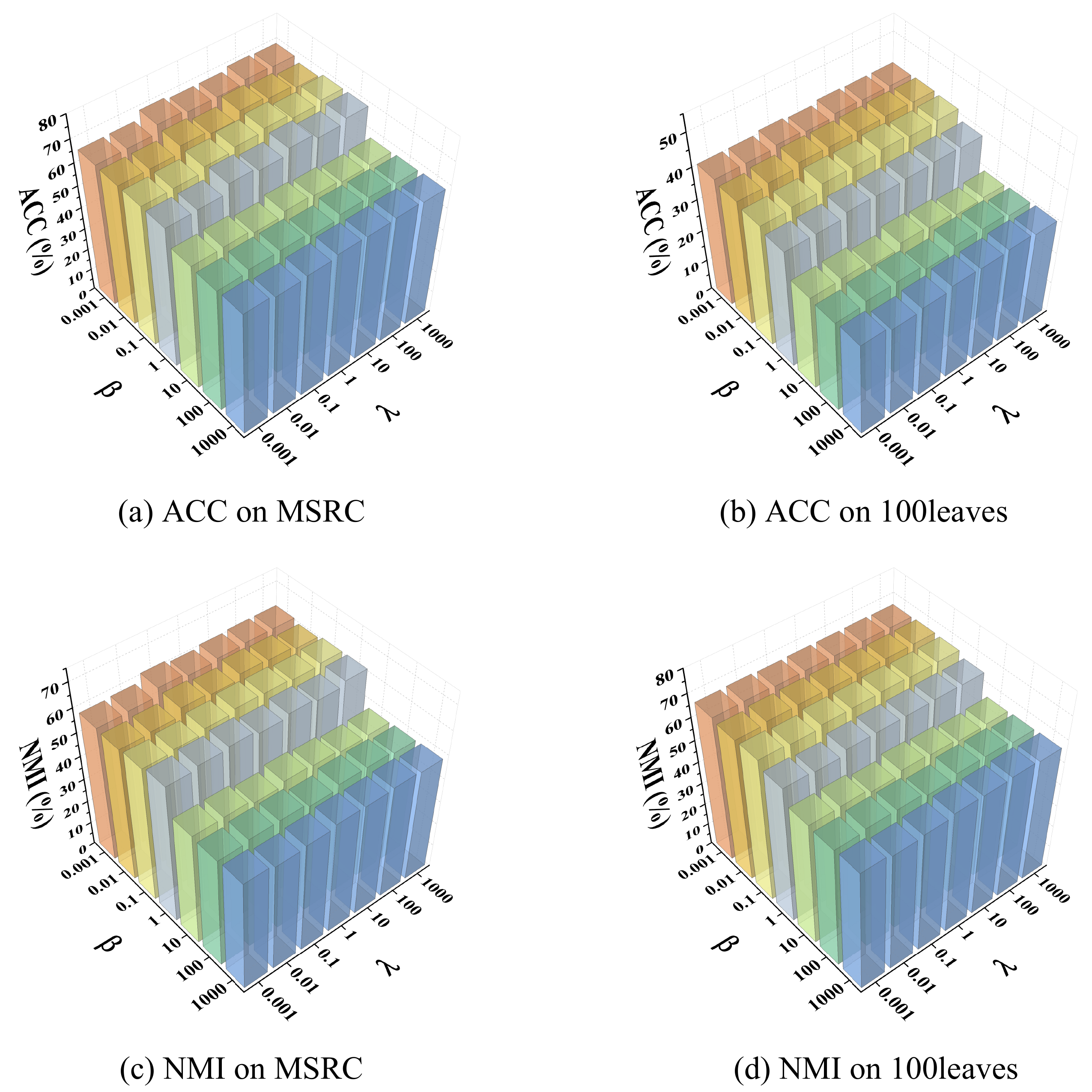}
\caption{ACC and NMI of CLIM-FS w.r.t different values of parameters $\lambda$ and $\beta$ on MSRC and 100leaves datasets.}\label{sensitivity}
\end{figure}

\begin{figure}[t]
\centering
\includegraphics[width=0.48\textwidth]{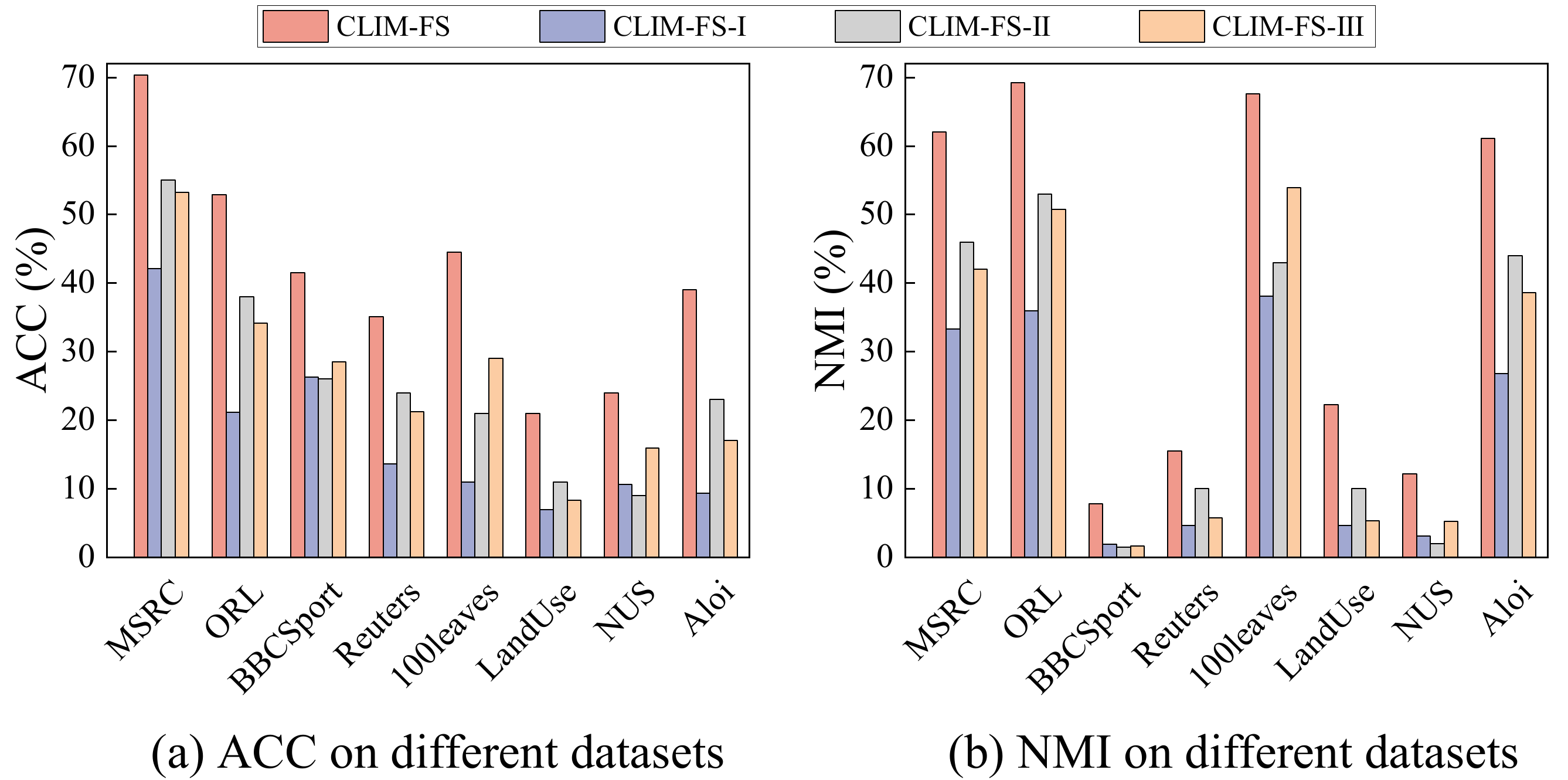}
\caption{Average performance comparison of CLIM-FS and its three variants in terms of ACC and NMI.}\label{ablation}
\end{figure}
\section{Conclusions}\label{sec:Conclusions}
In this paper, we proposed a novel IMUFS method, referred to as CLIM-FS. Unlike existing IMUFS approaches, which were limited to addressing view-missing issues, CLIM-FS tackled the more general mixed-missing problem by integrating feature selection with the imputation of missing views and variables within a unified learning framework. Moreover, CLIM-FS leveraged consensus cluster structures and cross-view local geometric structures to facilitate collaborative learning between feature selection and data imputation. We also provided a theoretical analysis to clarify this collaborative learning mechanism. Furthermore, we developed an iterative optimization algorithm with proven theoretical convergence to solve the proposed model. Experimental results on eight real-world datasets demonstrated that CLIM-FS outperformed state-of-the-art methods. In future work, we will extend CLIM-FS to handle streaming multi-view data, enabling incremental updates to the feature selection model and imputing missing data as new streams arrive, without requiring retraining from scratch.

\end{document}